\documentclass[twoside]{article}
\usepackage[utf8]{inputenc}
\usepackage[T1]{fontenc}

\usepackage[accepted]{aistats2026}
\usepackage[round]{natbib}
\renewcommand{\bibname}{References}
\renewcommand{\bibsection}{\subsubsection*{\bibname}}
\usepackage{subcaption}

\usepackage{multirow}
\usepackage{graphicx}
\usepackage{subcaption} 
\usepackage{booktabs}   

\usepackage{algorithmicx}
\usepackage{algorithm}
\usepackage{algpseudocode}
\usepackage{amsmath,amsthm,amsfonts,amssymb,amscd,mathtools,bm}
\usepackage{graphicx}
\usepackage{booktabs}
\usepackage[labelfont=bf]{caption}
\usepackage{xcolor}
\usepackage{nicefrac}

\newtheorem{theorem}{Theorem}

\newtheorem{proposition}{Proposition}

\newcommand{\E}{\mathbb{E}}
\DeclareMathOperator{\Var}{Var}
\newcommand{\follows}{\sim}

\DeclarePairedDelimiterX{\bgiven}[2]{[}{]}{#1\:\delimsize|\:#2}

\DeclarePairedDelimiterX{\setbuilder}[2]{\{}{\}}{#1\:\delimsize|\:#2}

\DeclarePairedDelimiterX{\pgiven}[2]{\lparen}{\rparen}{#1\:\delimsize|\:#2}

\usepackage{hyperref}
\usepackage{url}
\usepackage[capitalize]{cleveref}

\DeclareMathOperator*{\argmin}{arg\,min}

\begin{document}
\twocolumn[

\aistatstitle{Boosted GFlowNets: Improving Exploration via Sequential Learning}

\aistatsauthor{Pedro Dall'Antonia$^{1}$ \, Tiago da Silva$^{2}$ \, Daniel A. de Souza$^{3}$ \vspace{6pt} \\ \bf César L. C. Mattos$^{4}$ \, Diego Mesquita$^{1,5}$\vspace{8pt}}

\runningauthor{Pedro Dall'Antonia, Tiago da Silva, Daniel A. de Souza, César L. C. Mattos, Diego Mesquita}
  
\aistatsaddress{ ${}^1$Getulio Vargas Foundation \,  ${}^2$MBZUAI \, ${}^3$University College London \, ${}^4$Federal University of Ceará \, ${}^5$2$\delta$ AI} ]

\begin{abstract}
Generative Flow Networks (GFlowNets) are powerful samplers for compositional objects that, by design, sample proportionally to a given non-negative reward. Nonetheless, in practice, they often struggle to explore the reward landscape evenly: trajectories toward easy-to-reach regions dominate training, while hard-to-reach modes receive vanishing or uninformative gradients, leading to poor coverage of high-reward areas.
We address this imbalance with Boosted GFlowNets, a method that sequentially trains an ensemble of GFlowNets, each optimizing a residual reward that compensates for the mass already captured by previous models.
This residual principle reactivates learning signals in underexplored regions and, under mild assumptions, ensures a monotone non-degradation property: adding boosters cannot worsen the learned distribution and typically improves it.
Empirically, Boosted GFlowNets achieve substantially better exploration and sample diversity on multimodal synthetic benchmarks and peptide design tasks, while preserving the stability and simplicity of standard trajectory-balance training. \looseness=-1
\end{abstract}

\section{Introduction}

Generative Flow Networks (GFlowNets) learn stochastic policies over directed acyclic state graphs to sample structured objects in proportion to unnormalized terminal rewards \citep{bengio2021}. 
This framework makes it possible to discover diverse candidates even when feedback is only available at the end of a trajectory, and has shown promise in scientific discovery and combinatorial design. 
In practice, however, training on large, multimodal targets is often bottlenecked by \emph{exploration}: the learner rapidly concentrates on easy-to-reach modes while leaving hard-to-reach ones severely undercovered, limiting both coverage and generalization. \looseness=-1

This imbalance emerges even when the forward policy nominally has full support. Once easy-to-reach modes are discovered, on-policy sampling concentrates data collection there, driving the training loss down in those regions while further reinforcing their visitation. Trajectories leading to hard-to-reach modes are visited only sporadically, and their gradients shrink in proportion to their low visitation probability, creating a self-reinforcing loop that amplifies mode imbalance.
Off-policy updates alleviate this feedback to some extent by decoupling data collection from the current policy, yet they still fail to allocate sufficient learning signal to rarely visited regions, yielding high-variance or weakly informative updates. As a result, both regimes struggle to recover remote modes, especially under long horizons or sparse rewards, despite theoretical convergence guarantees.

\begin{figure*}[h]
  \centering
  \newcommand{\pw}{0.235\linewidth}
  \begin{subfigure}[t]{\pw}\centering
    \includegraphics[width=\linewidth]{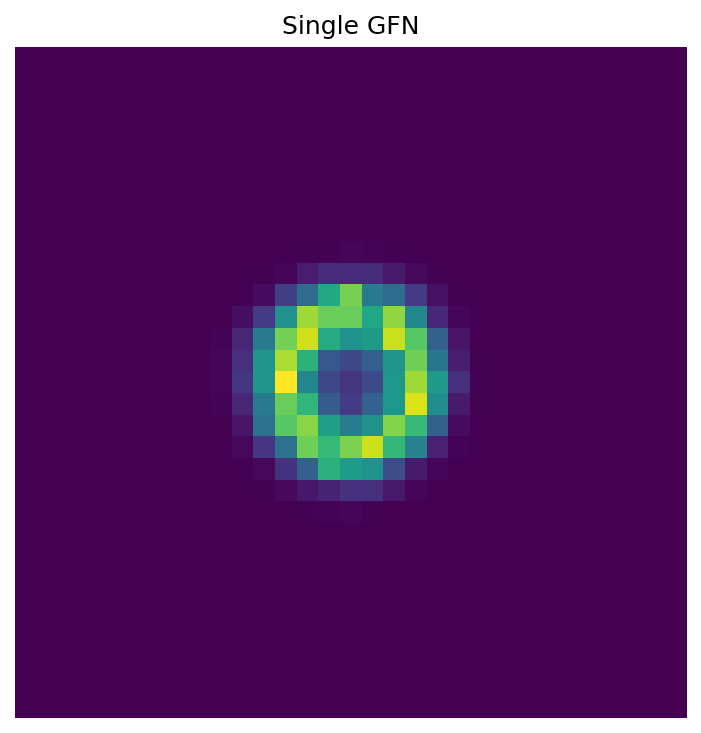}
    \caption{Single GFN}\label{fig:intu-baseline}
  \end{subfigure}\hfill
  \begin{subfigure}[t]{\pw}\centering
    \includegraphics[width=\linewidth]{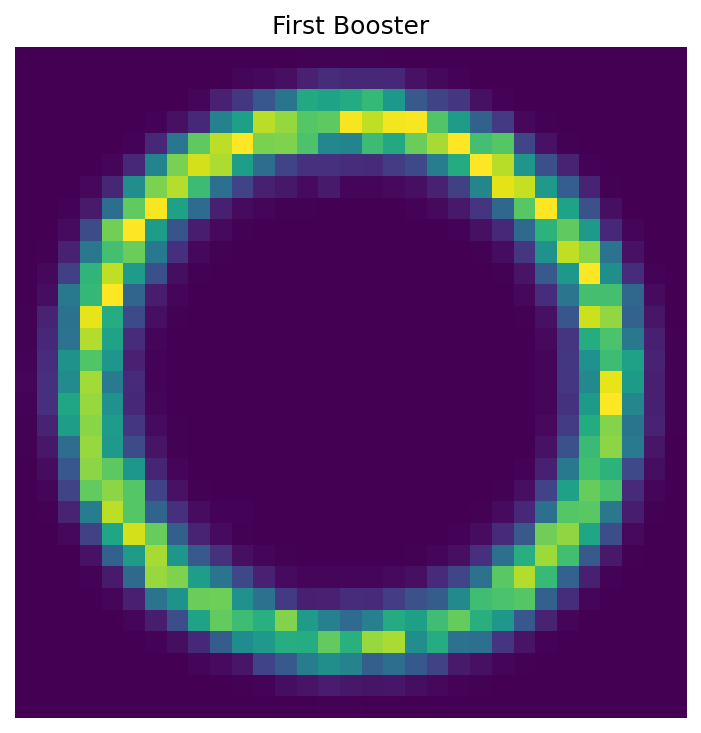}
    \caption{First Booster}\label{fig:intu-b1}
  \end{subfigure}\hfill
  \begin{subfigure}[t]{\pw}\centering
    \includegraphics[width=\linewidth]{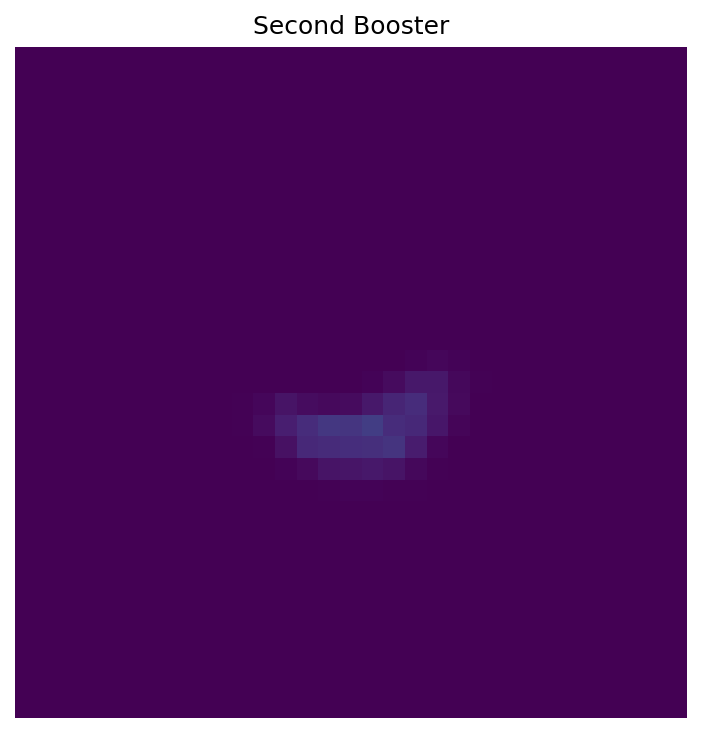}
    \caption{Second Booster}\label{fig:intu-b2}
  \end{subfigure}\hfill
  \begin{subfigure}[t]{\pw}\centering
    \includegraphics[width=\linewidth]{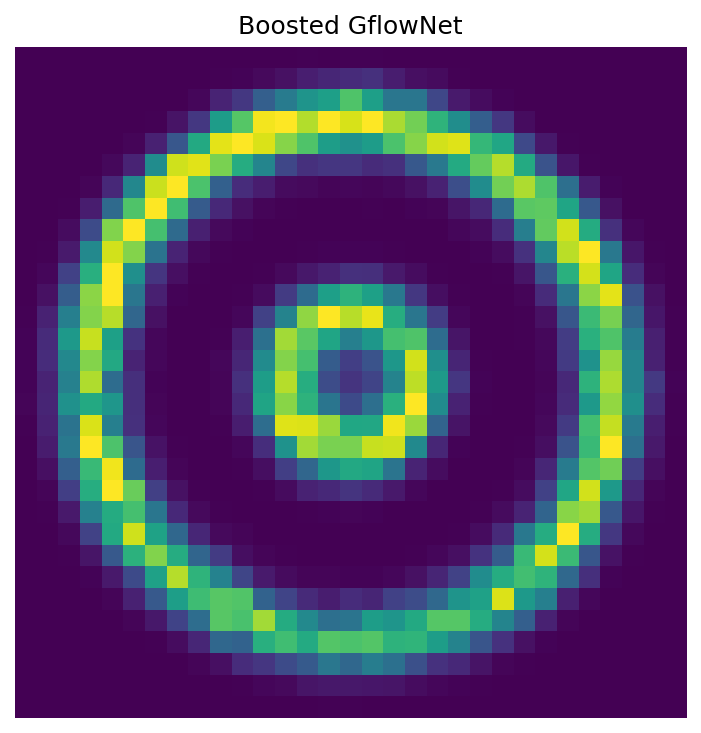}
    \caption{Final Ensemble}\label{fig:intu-ens}
  \end{subfigure}
  
  \caption{Illustration of Boosted GFlowNets on a multimodal target. 
  The single GFN covers only the nearest mode. 
  The first and second boosters progressively capture additional modes, while later boosters learn to allocate negligible flow. 
  Combined, the ensemble recovers the full target distribution.}
  \label{fig:intuition-evolution}
\end{figure*}

We propose \emph{Boosted GFlowNets} (BGFNs), a sequential training framework where each new stage is trained on the residual reward that remains after accounting for the distribution already induced by previous stages. 
This simple mechanism redistributes probability mass away from saturated regions and toward undercovered modes without modifying the core GFlowNet machinery or introducing auxiliary exploration modules. 
As illustrated in Figure~\ref{fig:intuition-evolution}, successive boosters progressively capture uncovered modes, while additional boosters introduced after convergence learn to allocate negligible flow, preserving correctness of the ensemble.

\noindent \textbf{Contributions.} Our main contributions are:
\begin{enumerate}
  \item We introduce \emph{Boosted GFlowNets}, a sequential training scheme that improves exploration by reallocating probability mass from easy-to-reach modes toward hard-to-reach ones, without altering the core GFlowNet framework.
  \item We prove that previously learned coverage remains stable, so additional boosters cannot degrade performance under mild assumptions; redundant boosters are effectively ignored.
  \item We validate the approach on multimodal synthetic benchmarks and on antimicrobial peptide generation with variable sequence lengths, showing that BGFNs achieve better mode coverage, diversity, and robustness compared to standard GFlowNet training.
\end{enumerate}

\vspace{-7 pt}

\section{Related Work}

\vspace{-5 pt}

GFlowNets face well-documented exploration bottlenecks on large, multimodal targets. Practical remedies span three main angles: \emph{(i) auxiliary exploration policies}, e.g., sibling-augmented schemes that collect novelty-seeking trajectories, often via intrinsic rewards such as Random Network Distillation, and re-label them off-policy for training the main learner \citep{madan2025-sa-gfn}; \emph{(ii) adaptive curricula}, where teacher--student mechanisms steer sampling toward under-covered regions identified by high loss or discrepancy signals \citep{kim2025-adaptive-teacher}; and \emph{(iii) loss-design perspectives}, where alternative regression losses (zero-avoiding families) are motivated through divergence analyses to bias training away from mode collapse and toward broader coverage \citep{hu2025-beyond-squared}. Scaling-oriented approaches decompose the state graph and train subgraphs asynchronously to enlarge visited regions under fixed budgets \citep{silva2025generalization}. Complementary theoretical works study conditions for correctness and distributional mismatch in finite-budget regimes, clarifying when coverage failures arise in practice \citep{silva2025when}.

Boosting Variational Inference \citep[BVI;][]{guo2016boosting, miller17boosting, locatello18boosting} generalizes classical boosting to probabilistic inference by iteratively constructing a mixture approximation to the target posterior. At each round, a new variational component is added to reduce the overall Kullback–Leibler divergence between the mixture and the true distribution. This procedure can be interpreted as performing functional gradient descent in the space of probability measures, where each component incrementally corrects the residual approximation error of the current mixture. Under mild smoothness and compactness assumptions on the variational family, BVI enjoys explicit convergence guarantees and systematically improves multimodal coverage by expanding the support of the approximation.

Boosted GFlowNets extend this idea to the domain of flow-based generative policies. Rather than updating a mixture density, each new GFlowNet is trained to correct the part of the target reward distribution that previous models have underrepresented. This iterative scheme redistributes probability mass from saturated regions toward underexplored modes, improving coverage while preserving the standard trajectory-balance formulation. By performing additive updates directly in flow space, BGFNs achieve monotone improvement in exploration and require no auxiliary policies or teacher modules.

\section{Background and Preliminaries}

Generative Flow Networks (GFlowNets, \citealt{bengio2021}) are probabilistic models designed to sample discrete objects $x \in \mathcal{X}$ with probability proportional to a non-negative reward function $R(x)$. 
Unlike standard generative models, which typically learn a normalized distribution, GFlowNets operate by training a pair of forward and backward policies over trajectories $\tau = (s_0 \to s_1 \to \dots \to s_T)$ in a directed acyclic graph (DAG) of states. 
The forward policy $P_F$ induces a distribution over terminal states, while the backward policy $P_B$ provides a stochastic inverse, together defining a system of ``flows'' that approximate $R(x)$. 
This formulation has been shown effective in domains where rewards are sparse, multimodal, or expensive to evaluate, making GFlowNets a promising tool for scientific discovery and structured generative modeling.

Formally, let $\Gamma$ denote the set of all valid trajectories from the initial state $s_0$ to any terminal state $x \in \mathcal{X}$. 
Each trajectory $\tau \in \Gamma$ is assigned a forward probability $P_F(\tau) = \prod_{t=0}^{T-1} P_F(s_{t+1}\mid s_t)$ and a corresponding backward probability $P_B(\tau \mid x) = \prod_{t=0}^{T-1} P_B(s_t \mid s_{t+1})$, where $s_T = x$. 
A GFlowNet is trained so that the marginal probability of sampling $x$ under the forward policy satisfies 
\begin{equation}
    P_F(x) \propto R(x), \quad \forall x \in \mathcal{X},
\end{equation}
with $Z = \sum_{x \in \mathcal{X}} R(x)$ acting as the normalizing constant. 
The objective of training is therefore to align the induced terminal distribution of $P_F$ with the unnormalized target defined by $R(x)$.

\paragraph{Training objectives.}
Several training objectives have been proposed to enforce the proportionality between $P_F(x)$ and $R(x)$, including \emph{Detailed Balance} (DB, \citealt{malkin2022trajectory}), \emph{Subtrajectory Balance} (SubTB, \citealt{madan2023learninggflownetspartialepisodes}). Our work builds on the widely used Trajectory Balance condition (TB, \citealt{malkin2022trajectory}). 
The TB objective introduces a scalar parameter $Z_{\theta}$ and enforces the following balance condition along full trajectories:

\begin{equation}
Z_{\theta}P_{F}(\tau) = R(x)P_{B}(\tau|x), \quad \forall \tau \in \Gamma
\label{eq:tb-condition}
\end{equation}

To achieve this, the model is trained by minimizing the squared log-ratio loss:

\begin{equation}
\mathcal{L}_{\text{TB}}(\tau) =
\Bigg(\log\frac{P_F(\tau)\, Z_\theta}{R(x)\, P_B(\tau \mid x)}\Bigg)^2
\label{eq:tb-loss}
\end{equation}

This loss enforces that the ratio of forward and backward trajectory probabilities, scaled by $Z_\theta$, matches the terminal reward $R(x)$. At convergence, minimizing this loss guarantees that the forward marginal satisfies $P_{F}(x)\propto R(x)$ provided that the support of the sampling policy covers all valid trajectories.

While the trajectory-balance objective provides a principled way to learn reward-proportional samplers, in practice its optimization can concentrate probability mass on a subset of easy-to-reach modes. In the next section, we revisit this limitation through the lens of flow decomposition, motivating the sequential training approach introduced in Boosted GFlowNets.

\section{Boosted GFlowNets}
\label{sec:method}


\subsection{Easy- vs.\ Hard-to-Reach Rewards}
\label{sec:easy-hard}

Let $Q$ denote the data-collection distribution over trajectories (on- or off-policy), and let $A \subseteq \Gamma$ be the subset of trajectories that terminate in \emph{easy-to-reach} modes; write $B := A^{c}$ for \emph{hard-to-reach} modes. Even if the forward sampler nominally has full support, complex DAG topologies and long horizons can concentrate $Q$ on $A$, starving $B$ of samples. This induces a systematic exploration imbalance that balance-style objectives may not correct in practice, yielding weak signals for hard modes and sluggish coverage.

For any per-trajectory training loss $\mathcal{L}(\tau;\theta)$, the training signal decomposes as
\begin{multline}
\label{eq:loss-decomp}
\mathbb{E}_{\tau \sim Q}\!\left[\mathcal{L}(\tau;\theta)\right]
= Q(A)\,\mathbb{E}\!\left[\mathcal{L}(\tau;\theta)\,\middle|\, \tau\in A\right] \\
{}+ Q(B)\,\mathbb{E}\!\left[\mathcal{L}(\tau;\theta)\,\middle|\, \tau\in B\right].
\end{multline}
As training progresses, the easy region $A$ is typically fit first, so $\mathbb{E}[\mathcal{L}(\tau;\theta)\mid \tau\in A]\!\to\!0$, and the objective reduces to
\begin{equation}
\mathbb{E}_{\tau \sim Q}\!\left[\mathcal{L}(\tau;\theta)\right]
\;\longrightarrow\;
Q(B)\,\mathbb{E}\!\left[\mathcal{L}(\tau;\theta)\,\middle|\, \tau\in B\right].
\end{equation}
Hence the residual learning signal is \emph{entirely scaled} by $Q(B)$: when $Q(B)\!\ll\!1$, gradients from hard modes become negligible even if those modes remain under-covered. Off-policy corrections trade bias for variance; small $Q(B)$ implies either a small effective sample size or unstable large weights, so the useful signal from $B$ is easily drowned out.

A natural thought is to “just sample terminals uniformly,” but this is often infeasible or ineffective: the set of terminals can be combinatorial/implicit and exponentially large (\emph{enumeration barrier}), many terminal objects are invalid under domain constraints so producing valid terminals uniformly is itself nontrivial (\emph{validity/constraints}). Consequently, in many environments we are effectively constrained to samplers induced by the DAG topology and the current policy, which can trap learning in easy regions.

Equation~\eqref{eq:loss-decomp} thus underscores a core design requirement: without a mechanism that \emph{reallocates learning pressure} toward the hard region $B$, the contribution of those modes to optimization remains negligible.

The discussion above suggests that we should treat the mass already captured in easy regions as \emph{explained} signal and reallocate optimization toward what remains under-covered. Our key observation is that, under Trajectory Balance, the contribution of well-learned regions becomes effectively deterministic: the model induces a near-exact estimate of the terminal reward wherever flows are matched. We leverage this induced estimate as a control variate to define a \emph{residual target} for the next stage.

\subsection{Induced Reward and Residual Principle}
\label{sec:residual-principle}

Under Trajectory Balance (TB), regions that are already well learned produce near-deterministic per-trajectory estimates of the terminal reward. We make this precise by defining the induced estimator and stating a zero-variance property that motivates our residual training scheme.

Let $\Gamma_x$ denote the set of trajectories that terminate at $x$. Throughout this section, we work with trained (hence fixed) GFlowNets and omit parameter dependence from the notation. For any GFlowNet $\mathfrak g := (Z, P_F, P_B)$, terminal $x$, and trajectory $\tau\in\Gamma_x$, define the per-trajectory estimator
\begin{equation}
\widehat{R}_{\mathfrak g}(x;\tau)
:= Z\,\frac{P_F(\tau)}{P_B(\tau\mid x)}.
\end{equation}

The corresponding induced terminal estimate is
\begin{equation}
\label{eq:old-reward-expectation}
\widehat{R}_{\mathfrak g}(x)
  := \mathbb{E}_{\tau \sim P_B(\cdot\mid x)}
     \big[\widehat{R}_{\mathfrak g}(x;\tau)\big].
\end{equation}

\begin{theorem}[Zero variance at optimum]
\label{thm:Zero-variance}
Fix a trained GFlowNet $\mathfrak g := (Z, P_F, P_B)$, and define the support of its forward policy's terminal distribution as
\[
S := \{x \in \mathcal{X} : P_F(x) > 0\}.
\]
Assume that the Trajectory Balance loss has zero expectation,
\[
\mathbb{E}_{\tau \sim P_F}\!\big[\mathcal{L}_{TB}(\tau)\big] = 0.
\]
Further assume that, for every terminal $x \in S$, the backward policy $P_B(\cdot\mid x)$ and the restriction of the forward policy $P_F(\cdot)|_{\Gamma_x}$ are mutually absolutely continuous. Then, for every terminal $x \in \mathcal{X}$,
\[
\widehat{R}_{\mathfrak g}(x) = R(x)\,\mathbb{I}[x \in S]
\]
and
\[
\Var_{\tau \sim P_B(\cdot\mid x)}\!\big[\widehat{R}_{\mathfrak g}(x;\tau)\big] = 0,
\]
where $\mathbb{I}[\cdot]$ is the indicator function.
\end{theorem}


This yields the residual principle: given a frozen predictor that already matches parts of the target, its induced terminal estimate can be used to define a residual target. The next stage is then trained to capture this residual, thereby reallocating learning pressure toward under-covered regions while preserving coverage in regions that are already matched. This principle is operationalized in our boosted loss, detailed in the next subsection.

For an ensemble of trained GFlowNets $\mathcal{G}=\{\mathfrak g_1,\dots,\mathfrak g_N\}$, we aggregate the induced estimates of its members via
$$
\widehat{R}_{\mathcal{G}}(x) := \sum_{n=1}^N \widehat{R}_{\mathfrak g_n}(x),
$$
To lighten notation, once the relevant trained model (single stage or ensemble) is clear from context, we omit the subscripts and simply write $\widehat{R}(x)$ and $\widehat{R}(x;\tau)$.

\subsection{Boosted Trajectory Balance Loss}
\label{sec:boosted-loss}

Our sequential training framework is motivated by a simple observation: the overall objective is to enforce the Trajectory Balance (TB) condition (Eq.~\eqref{eq:tb-condition}) over the set of all possible trajectories, $\Gamma$.

When a GFlowNet minimizes the expected Trajectory Balance (TB) loss, $\mathbb{E}_{\tau \sim P_{\mathrm{F}}}[\mathcal{L}_{\text{TB}}(\tau)]$ to zero the TB condition (Eq.~\eqref{eq:tb-condition}) is satisfied for all trajectories within the support of the forward policy, $P_{\mathrm{F}}$.
Let $S$ be the set of terminal states covered by a converged past GFlowNet, and let $\Gamma_{S} \subset \Gamma$ be the set of trajectories ending in $S$.
The remaining learning objective for a new GFlowNet is thus to enforce the same balance condition on the complement set of trajectories, $\Gamma \setminus \Gamma_{S}$, which lead to the under-covered modes.

A principled way to enforce this new, restricted objective is to reformulate the target reward itself.
We can formally decompose the total target $R(x)$ into two components: the reward already captured by the \textbf{past model}, and the reward that remains.
\begin{equation}
    R(x) = \underbrace{R(x) \cdot \mathbb{I}[x \in S]}_{\text{Captured Reward}} + \underbrace{R(x) \cdot \mathbb{I}[x \notin S]}_{\text{Residual Reward}}
\end{equation}
As established in our analysis of Theorem~\ref{thm:Zero-variance}, the "Captured Reward" term is precisely the expected estimate produced by the trained GFlowNet, which we denote $\widehat{R}(x)$~(cf. Eq.~\eqref{eq:old-reward-expectation}).
This allows us to rewrite the equality above as:
\begin{equation}
    R(x) - \widehat{R}(x) = R(x) \mathbb{I}[x \notin S]
\end{equation}
Enforcing the TB condition on the complement is thus equivalent to enforcing a \emph{new} balance condition over the \emph{entire} set $\Gamma$, where the target is this $R(x) - \widehat{R}(x)$.

This decomposition suggests two equivalent balance conditions for the next stage, expressed in terms of the reward-induced  estimator $\widehat{R}_{\theta}(x;\tau) \!\!\;:= Z_\theta \,\frac{P_F^\theta(\tau)}{P_B^\theta(\tau \mid x)}.$
The new model can be trained to satisfy, either:

\begin{enumerate}
    \item \textbf{Target-Residual (TR):} the new stage directly matches the residual target,
    \begin{equation}
        \widehat{R}_{\theta}(x;\tau) = R(x) - \widehat{R}(x),
        \label{eq:tr-condition}
    \end{equation}

    \item \textbf{Flow-Additive (FA):} equivalently, its contribution adds to the frozen model to recover the full reward,
    \begin{equation}
        \widehat{R}_{\theta}(x;\tau) + \widehat{R}(x) = R(x),
        \label{eq:fa-condition}
    \end{equation}
\end{enumerate}

While derived here for a single past model, this same principle applies when $\widehat{R}$ represents the combined flow of a pre-existing ensemble.

These two strategies are unified into a single \textbf{Boosted Trajectory Balance Loss}, controlled by $\alpha \in [0,1]$ and the Monte Carlo budget $k$:
\begin{equation}
\mathcal{L}_{\text{boost}}^{(k)}(\tau)
= \left(
\log \left[
\frac{\widehat{R}_{\theta}(x;\tau) + \alpha \widehat{R}_k(x)}
     {R(x) - (1-\alpha)\widehat{R}_k(x)}
\right]
\right)^{2}.
\label{eq:boosted-loss}
\end{equation}
We omit the dependence on $k$ when it is clear from context or irrelevant to our analysis.

Here, $\widehat{R}_k(x)$ is a $k$-sample Monte Carlo estimate of the frozen model's induced terminal estimate $\widehat{R}(x)$ (Eq.~\eqref{eq:old-reward-expectation}). In the ensemble case, we draw $k$ trajectories from each member's backward policy, average within each member, and sum the resulting estimates:
\begin{equation}
\label{eq:mc-Rhat}
\widehat{R}_k(x) \!
:= \!  \sum_{n=1}^N\frac{1}{k}\sum_{i=1}^k \widehat{R}_{\mathfrak g_n}(x;\tau_{n,i}),
\tau_{n,i} {\sim} P_B^{(n)}(\cdot\mid x).
\end{equation}
The following theorem formalizes the desirable properties of this loss at its boundaries.
\begin{theorem}[Correctness of the Boosted Loss]
\label{thm:boosted-loss-correctness}
For a terminal $x$, let $\widehat{R}(x)$ be an estimator of the target reward $R(x)$. Then:
\begin{enumerate}
    \item \textbf{(No Degradation)} If $\widehat{R}(x)=R(x)$, then the stationary points of $\mathcal{L}_{\text{boost}}$ (Eq.~\eqref{eq:boosted-loss}) have $Z_{\theta}=0$.
    \item \textbf{(Residual Focus)} If $\widehat{R}(x)=0$, then the stationary points of $\mathcal{L}_{\text{boost}}$ are the same as those of the standard TB loss, $\mathcal{L}_{\text{TB}}$ (Eq.~\eqref{eq:tb-loss}).
\end{enumerate}
\end{theorem}
\begin{figure*}[h]
  \centering
  \begin{subfigure}[t]{0.32\textwidth}
    \centering
    \includegraphics[width=\linewidth]{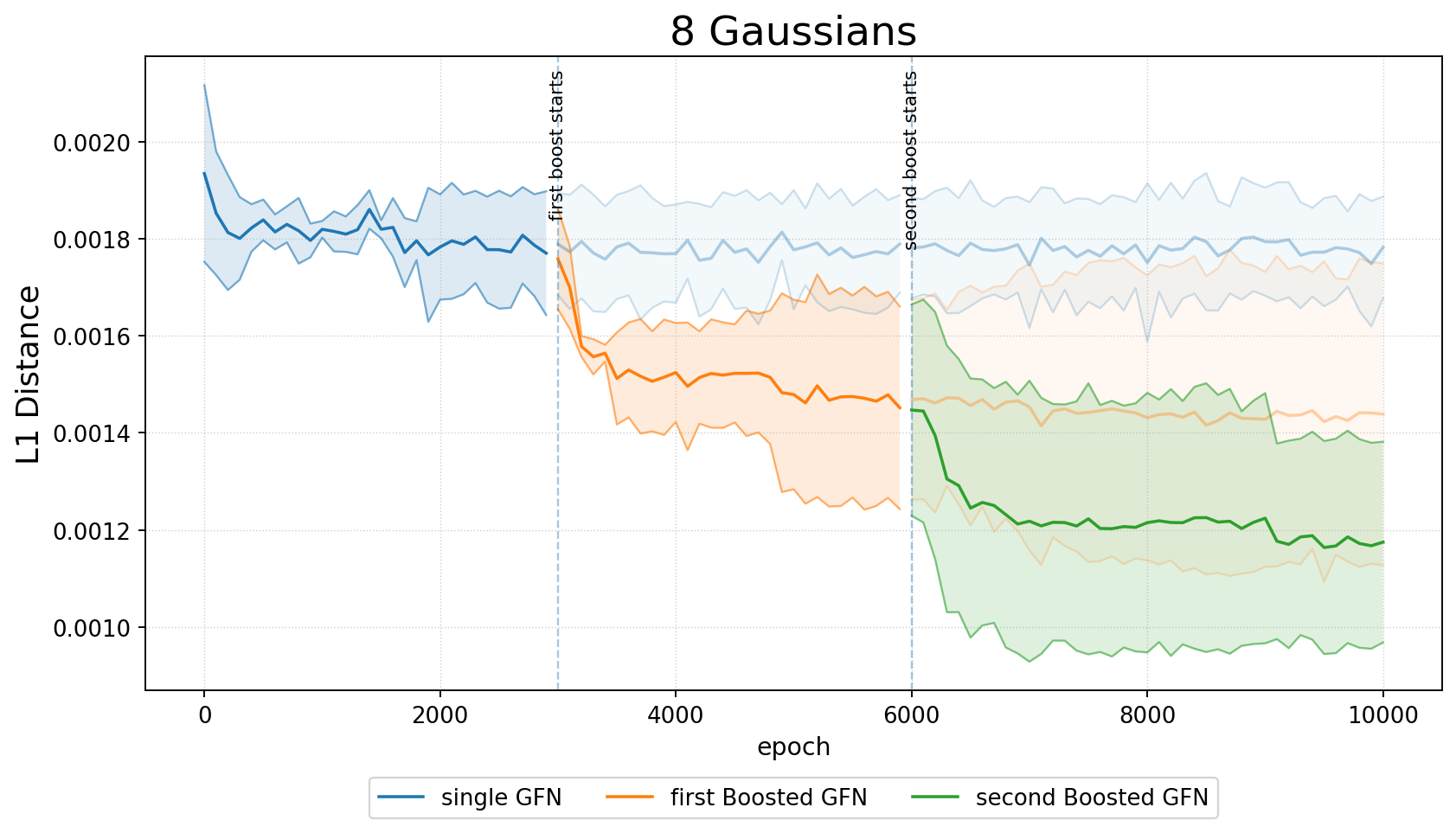}
    \caption{\textsc{Eight-Gaussians}}
    \label{fig:curve-8g}
  \end{subfigure}
  \hfill
  \begin{subfigure}[t]{0.32\textwidth}
    \centering
    \includegraphics[width=\linewidth]{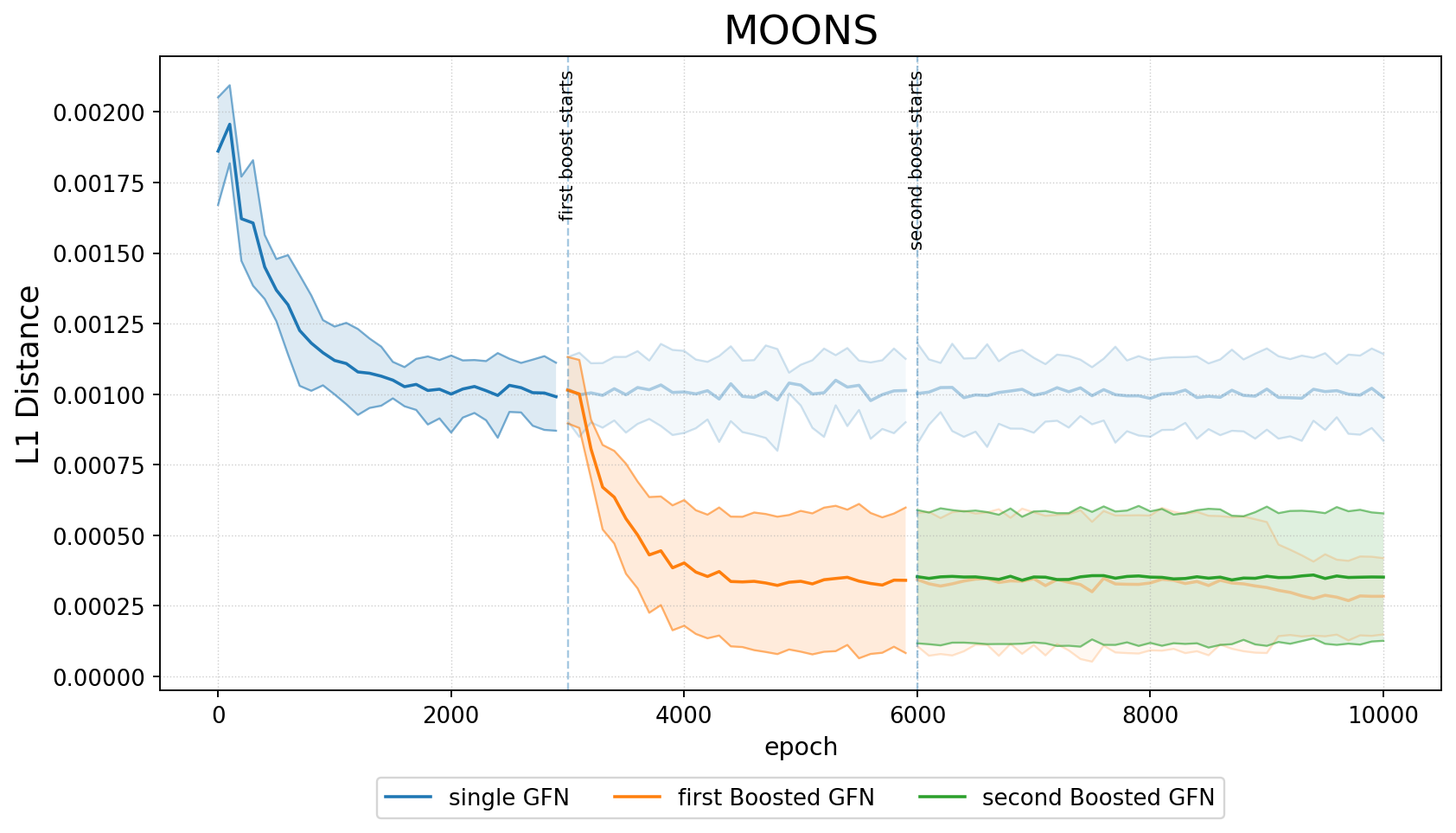}
    \caption{\textsc{Moons}}
    \label{fig:curve-moons}
  \end{subfigure}
  \hfill
  \begin{subfigure}[t]{0.32\textwidth}
    \centering
    \includegraphics[width=\linewidth]{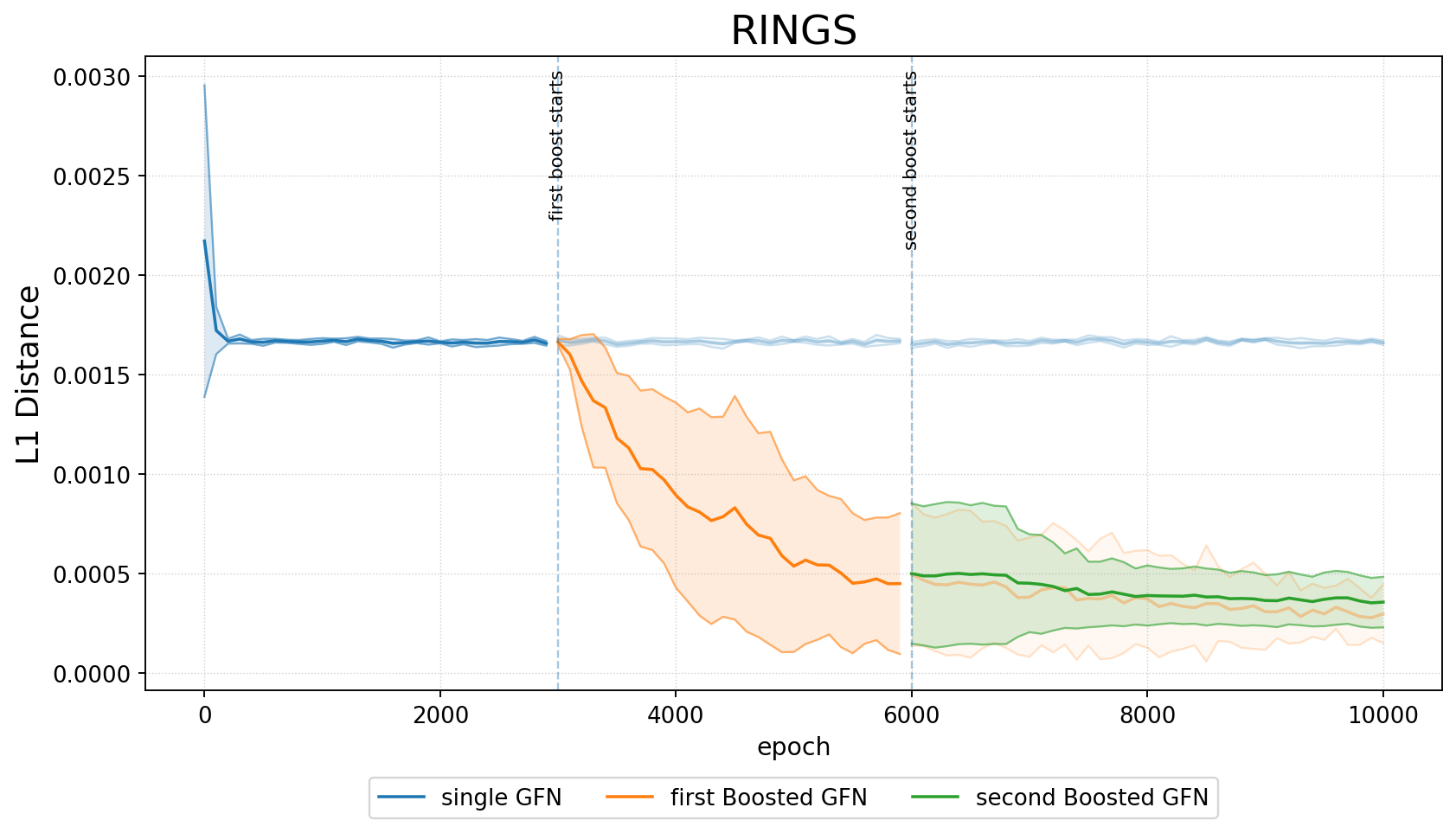}
    \caption{\textsc{Rings}}
    \label{fig:curve-rings}
  \end{subfigure}

  \caption{\textbf{Learning curves on synthetic targets.}
  Solid lines show the mean \(L_1\) across seeds; shaded bands denote \(\pm1\) std.
  Vertical dashed lines (3k and 6k epochs) indicate booster activations for BGFN(2) and BGFN(3).
  The single-GFN baseline (TB) plateaus once easy modes are fit; boosted stages keep reducing error by reallocating mass toward hard modes.}
  \label{fig:curves-synth}
\end{figure*}

Setting $\alpha\!=\!0.0$ recovers the Target-Residual (TR) variant (Eq.~\eqref{eq:tr-condition}), while setting $\alpha=1.0$ recovers the Flow-Additive (FA) variant (Eq.~\eqref{eq:fa-condition}).

The TR variant is mathematically \textbf{undefined} when $\widehat{R}(x) \ge R(x)$, as the denominator in Eq.~\eqref{eq:boosted-loss} becomes zero or negative.
This problematic case occurs precisely when the old ensemble has perfectly learned (or over-estimated) a mode. To mitigate this, we can choose $\alpha$ such that it remains as close to 0 while maintaining numerical stability.



\subsection{Sampling from the ensemble}
\label{sec:sampling}

After training an ensemble of $N$ GFlowNets, $\{(Z_i, P_{i,\mathrm{F}}, P_{i,\mathrm{B}})\}_{i=1}^{N}$, a procedure is needed to sample a terminal state $x \in \mathcal{X}$ that is proportional to the total target reward $R(x)$.

The core principle of our framework is that each stage $i$ is trained to capture a residual component of the total reward.
Consequently, its learned partition function, $Z_i$, serves as an estimate of the total mass of that component.
The natural way to sample from the full ensemble is therefore to treat it as a mixture model, where the probability of selecting stage $i$ is proportional to its learned mass, $Z_i$.
The following theorem, proven in the appendix, formalizes that this sampling procedure correctly recovers the target distribution.

\begin{theorem}[Correctness of the Ensemble Sampling Process]
\label{thm:sampling-correctness}
Given an ensemble of $N$ GFlowNets that satisfy the boosted loss objective over a full support distribution, define the mixture probability of sampling a terminal state $x$ as:
\begin{equation}
    \hat{p}(x) \coloneqq \sum_{i=1}^{N} \frac{Z_{i}}{\sum_{j=1}^{N}Z_{j}} \sum_{\tau\in\Gamma_{x}} P_{i,\mathrm{F}}(\tau).
    \label{eq:ensemble-prob}
\end{equation}
Then, the induced probability mass is proportional to the target reward: $\hat{p}(x) \propto R(x)$.
\end{theorem}

This theorem gives rise to a simple, practical two-step sampling algorithm:
\begin{enumerate}
    \item \textbf{Select a stage:} Sample an index $i \in \{1, ..., N\}$ from a categorical distribution with probabilities proportional to the learned partition functions, $p(i) \propto Z_i$.
    \item \textbf{Sample from the stage:} Execute the forward policy $P_{i,\mathrm{F}}$ of the selected stage to sample a terminal state $x$.
\end{enumerate}


\subsection{Relationship to Boosting VI}

Notably, Boosted GFlowNets are also tightly connected to the framework of boosting variational inference \citep[BVI;][]{guo2016boosting, miller17boosting, locatello18boosting}. To understand this relationship, we recall that BVI iteratively builds a mixture distribution by minimizing its Kullback-Leibler divergence to the target. That is, for each $t > 1$,
\begin{equation} \label{eq:bvis}
    q^{\star}, \beta^{\star} = \argmin_{q \in \mathcal{H}, \ \beta \in [0, 1]} \mathcal{D}_{\mathrm{KL}} \left[ (1 - \beta) \cdot p^{(t)} + \beta \cdot q || \pi \right],
\end{equation}

for a tractable family $\mathcal{H}$ of distributions. Then, $p^{(t + 1)} = (1 - \beta^{\star}) \cdot p^{(t)} + \beta^{\star} \cdot q^{\star}$. Clearly, each $p^{(t)}$ is a mixture of elements of $\mathcal{H}$ for $t > 1$ \citep{locatello2018boosting}. \looseness=-1

Drawing on this, we demonstrate below that the expected gradient of our boosted loss function ($L_{\text{boost}}$) under the ensemble distribution equals the gradient of the KL divergence in \Cref{eq:bvis} with respect to the parameters of $q$ (up to a multiplying constant). 

\begin{proposition} \label{prop:bvis}
    Under the notations of \Cref{eq:boosted-loss}, let $\widehat Z$ and $Z_{\theta}$ be the normalizing constants for $\widehat{R}(x)$ and $\hat{R}_{\theta}(x; \tau)$. 
    Also, let $\beta = \nicefrac{Z_{\theta}}{\widehat Z + Z_{\theta}}$ be the mixture weight in \Cref{eq:bvis}. 
    Then, define $\widehat p(\tau) \!\propto\! \widehat{R}(x) p_B(\tau | x)$ and $p_{\mathrm{tgt}}(\tau) \!=\! R(x) p_B(\tau | x)$ as the trajectory-level distributions induced by our current and target models, and let $p_M(\tau) = (1 - \beta) \cdot \widehat p(\tau) + \beta \cdot p_{F}(\tau)$ be the corresponding mixture distribution.
    In this scenario,
    \begin{equation}
        \mathbb{E}_{\tau \sim p_M} \left[ \nabla_{\theta} L_{\mathrm{boost}}(\tau) \right] = 2 \cdot \nabla_{\theta} D_{\mathrm{KL}} [ p_M(\tau) || p_{\mathrm{tgt}}(\tau) ] 
    \end{equation}
    when $\alpha$ is set to $1$ in \Cref{eq:boosted-loss}.  
\end{proposition}

Importantly, \Cref{prop:bvis} expands the connection between GFlowNets and VI \citep{malkin2023gflownets} and provides a formal motivation for our method's name. \looseness=-1

\section{Experiments}
\label{sec:experiments}

\begingroup
\setlength{\tabcolsep}{3.0pt}%
\renewcommand{\arraystretch}{0.95}%
\begin{table*}[h]
    \caption{Last-epoch $L_1$ distance. Means (1 sig. digit) and std (1 sig. digit) shown as mean(std) using compact scientific notation. All runs use the same size ($n=6$).
    \label{tab:l1_tasks_ultracompact_sciShort}
  Abbrev.: SG = Single GFlowNet; B2--FA = Boosted (2 models, flow-additive); B3--FA = Boosted (3 models, flow-additive).
  Lower is better. \textbf{Bold} = best per noise $\varepsilon$. \textcolor{red}{Red} = best overall within each task.}

  \centering
  \footnotesize
  \begin{tabular}{llcccccc}
    \toprule
    \multicolumn{2}{l}{\textbf{Task / Model}} & \multicolumn{6}{c}{\(\varepsilon\)} \\
    \cmidrule(lr){3-8}
     &  & 0 & 0.1 & 0.2 & 0.3 & 0.4 & 0.5 \\
    \midrule
\multirow{3}{*}{\textbf{8g}} & SG & 1.8e-3(1.7e-5) & 1.8e-3(1e-4) & 1.2e-3(1.6e-4) & \textbf{1e-3(2.9e-4)} & 1.2e-3(2.8e-4) & 1.7e-3(5.9e-4) \\
 & B2--FA & 1.6e-3(2e-5) & 1.4e-3(3.1e-4) & 1.2e-3(2.4e-4) & 1.1e-3(1.6e-4) & 1.3e-3(2.4e-4) & 1.5e-3(1e-4) \\
 & B3--FA & \textbf{1.3e-3(1.6e-4)} & \textbf{1.2e-3(2.1e-4)} & \textbf{\textcolor{red}{9.8e-4}}(2.1e-4) & 1e-3(8.7e-5) & \textbf{1.2e-3(1.6e-4)} & \textbf{1.4e-3(7.4e-5)} \\
\midrule
\multirow{3}{*}{\textbf{rings}} & SG & 1.7e-3(1.2e-5) & 1.7e-3(1.1e-5) & 1.7e-3(9.3e-6) & 1.7e-3(2.3e-5) & 1.6e-3(3.8e-5) & 1.6e-3(3.7e-5) \\
 & B2--FA & \textbf{\textcolor{red}{2.6e-4}}(6.6e-5) & \textbf{3e-4(1.5e-4)} & \textbf{9.4e-4(8e-4)} & 9.6e-4(7.6e-4) & \textbf{9.7e-4(7.6e-4)} & 1.2e-3(6.8e-4) \\
 & B3--FA & 3.5e-4(2e-4) & 3.5e-4(1.3e-4) & 9.5e-4(7.9e-4) & \textbf{8e-4(7.1e-4)} & 9.7e-4(7.5e-4) & \textbf{1.2e-3(6.7e-4)} \\
\midrule
\multirow{3}{*}{\textbf{moons}} & SG & 1e-3(5.9e-6) & 1e-3(1e-5) & 9.9e-4(1.5e-4) & 1.1e-3(1.9e-4) & 1.1e-3(1.2e-4) & 1.2e-3(1.1e-4) \\
 & B2--FA & \textbf{1.9e-4(1.9e-5)} & \textbf{\textcolor{red}{1.9e-4}}(3.3e-5) & \textbf{2.8e-4(1.4e-4)} & 4.2e-4(1e-4) & 6.3e-4(1.1e-4) & 6e-4(9.1e-5) \\
 & B3--FA & 2.1e-4(3.8e-5) & 2.2e-4(3.4e-5) & 3.5e-4(2.3e-4) & \textbf{3.4e-4(8.1e-5)} & \textbf{4e-4(2.5e-4)} & \textbf{2.9e-4(1.8e-4)} \\
\midrule
    \bottomrule
  \end{tabular}

\end{table*}

\endgroup 

Our experimental evaluation aims to answer three main questions:
\begin{description}
    \setlength{\itemsep}{0pt}
    \item[\textbf{(Q1)}] Can Boosted GFlowNets (BGFNs) successfully explore modes that are hard to reach due to the inductive bias of standard GFlowNets towards easy-to-reach regions?
    \item[\textbf{(Q2)}] Do BGFNs correctly learn the target distribution without introducing bias, and how is this affected by on-policy vs. off-policy training and estimator noise?
    \item[\textbf{(Q3)}] When additional boosters are instantiated after the distribution is already well fitted, can BGFNs effectively ignore them without degrading performance?
\end{description}
To address these questions, we consider four tasks. Three of them are synthetic multimodal environments on a two-dimensional grid, where the GFlowNet is trained to sample proportionally to a reward landscape defined over a $31 \times 31$ lattice. The fourth task is a real-world application in computational biology, namely the generation of antimicrobial peptides (AMPs). Together, these benchmarks allow us to probe exploration challenges, sensitivity to noise, and practical utility in a scientific discovery setting.

\subsection{Synthetic Multimodal Environments}
\label{sec:synthetic}

In the grid tasks, the environment is a $31 \times 31$ lattice centered at the origin. Each episode consists of exactly 30 steps, where at each step the agent can move up, down, left, or right, or remain in place. Terminal states correspond to positions on the grid, and the reward function is multimodal, with three distinct landscapes (EIGHT-GAUSSIANS, MOONS, RINGS) chosen to test the limitations of standard GFlowNets~\ref{fig:curves-synth}. The main challenge stems from the combinatorial number of trajectories leading to each state, which makes distant modes extremely difficult to reach in practice.

\paragraph{Training protocol.}
We organized the experiments into three settings: (i) a single GFlowNet trained for 10,000 epochs (SINGLE), (ii) a baseline trained for 3,000 epochs followed by a first booster trained for an additional 7,000 epochs (BGFN-2), and (iii) the same baseline and first booster up to 6,000 epochs, after which a second booster is introduced and trained for 4,000 epochs (BGFN-3). To control for stochasticity, all models were trained with identical random seeds across conditions.

\subsubsection{Results: BGFNs Overcome Exploration Bottlenecks and are Robust}

Our results on the synthetic environments show a clear and consistent pattern, visualized in Figure~\ref{fig:intuition-evolution} Figure~\ref{fig:curves-synth} and Table~\ref{tab:l1_tasks_ultracompact_sciShort}:

\begin{itemize}
    \item \textbf{Single GFlowNets fail to explore.} Across all three multimodal landscapes, the single GFlowNet (SG) baseline quickly converged to easy-to-reach modes but systematically failed to cover distant, high-reward modes. This is evident in the L1 distance between the distribution induced on terminal states by the GFlownet and the target distribution (Figure~\ref{fig:curves-synth}, blue line), which plateau early as the model fails to discover remote modes. \looseness=-1

    \item \textbf{BGFNs successfully explore hard modes.} By contrast, BGFNs consistently discovered and allocated probability mass to the hard-to-reach modes. The L1 distance shows the first booster (BGFN-2, orange line) breaking the baseline's plateau and continuing to reduce (Figure~\ref{fig:curves-synth}. This process is illustrated in Figure~\ref{fig:intuition-evolution}, where the single GFN captures only the inner ring, and the first booster correctly learns the residual (the outer ring). \looseness=-1
    
    \item \textbf{BGFNs are robust to redundancy.} Importantly, the introduction of a second, unnecessary booster (BGFN-3) did not degrade performance. Once the distribution was already well captured by the first booster, the second booster (Figure~\ref{fig:curves-synth}, green line) effectively learned to allocate negligible flow, leaving the ensemble performance unchanged. This confirms our claim (Q3) that BGFNs can adaptively "ignore" redundant models, avoiding destabilization.
\end{itemize}

\begin{figure*}[h]
  \centering
  \begin{subfigure}[t]{0.44\textwidth}
    \centering
    \includegraphics[width=\linewidth]{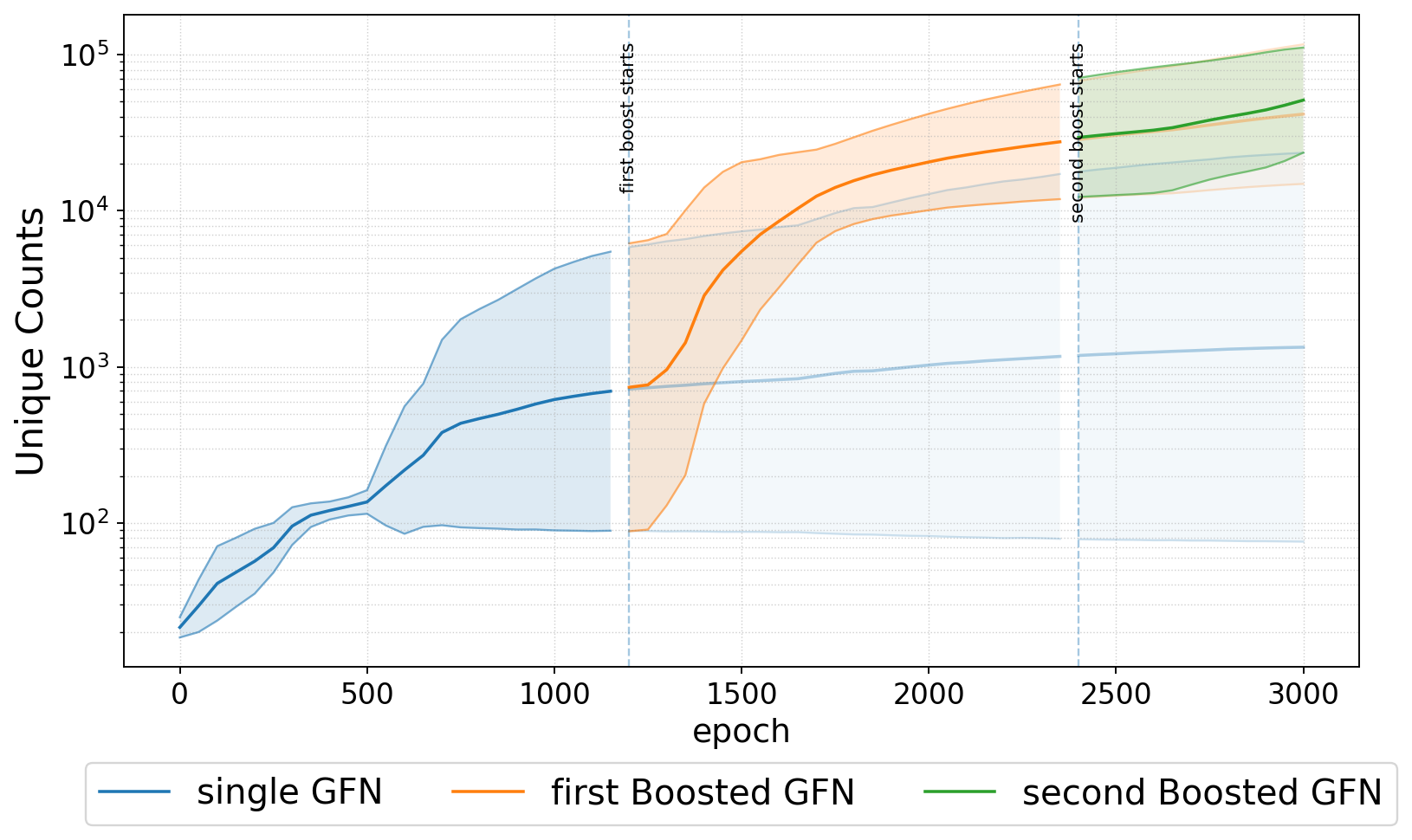}
    \label{fig:curve-peptides1}
  \end{subfigure}
  \begin{subfigure}[t]{0.44\textwidth}
    \centering
    \includegraphics[width=\linewidth]{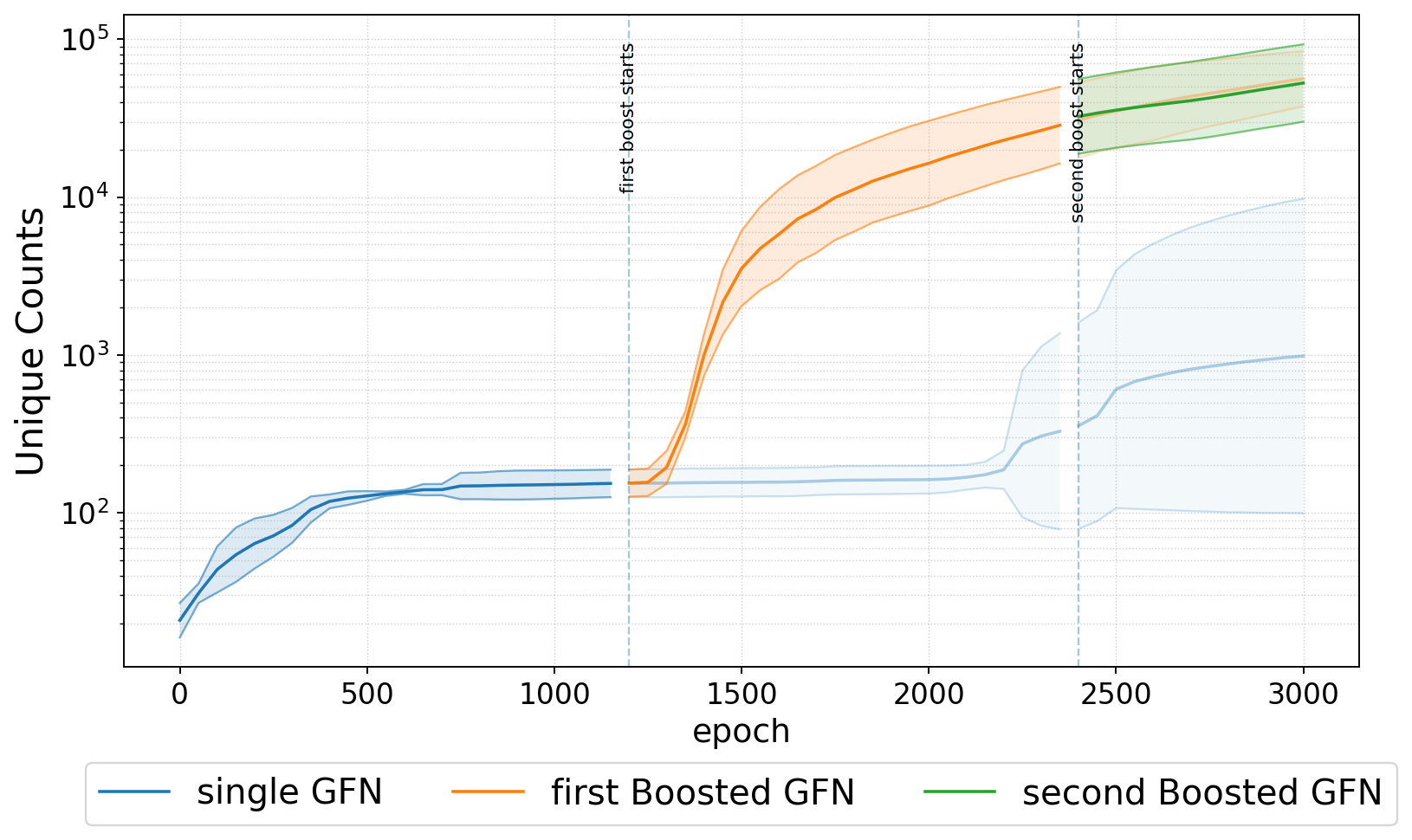}
    \label{fig:curve-peptides2}
  \end{subfigure}
\caption{\textbf{Unique predicted-resistant peptides across noise levels.}
Every 50 epochs we sample 1{,}000 peptides from the current policy and
\emph{accumulate} the number of \emph{unique} sequences whose predicted activity
is at least $0.94$ for at least one microorganism. Curves show the cumulative
count over epochs for a single GFN (blue) and its boosted counterparts (orange,
green); shaded regions denote $\pm 1$ standard deviation across seeds. The
$y$-axis is logarithmic. Vertical dashed lines mark activation of the first and
second boosters. \textbf{Left:} off-policy $\varepsilon=0.2$; \textbf{Right:} off-policy
$\varepsilon=0.3$.}
  \label{fig:uniq-peptides}
\end{figure*}

\begin{table*}[h]
  \centering
  \caption{Unique predicted-resistant peptides ($\mu \pm \sigma$, $n=5$).
  Abbreviations: SG = Single GFlowNet; B2 = Boosted GFlowNet (2 models); 
  B3 = Boosted GFlowNet (3 models).
  TR = target-residual ($\alpha=0.0$); FA = flow-additive ($\alpha=1.0$).
  Bold = best per noise level $\varepsilon$; \textcolor{red}{Red} = best overall in the table.}
  \label{tab:uniq-summary}
  \begin{tabular}{lcccc}
    \toprule
    \textbf{Model} & \multicolumn{4}{c}{$\varepsilon$} \\
    \cmidrule(lr){2-5}
                   & 0.0 & 0.1 & 0.2 & 0.3 \\
    \midrule
    SG    & 32.6 $\pm$ 11.4 & 70362.8 $\pm$ 104681.2 & 14872.0 $\pm$ 25536.5 & 4958.0 $\pm$ 6594.3 \\
    B2-TR & \textbf{2672.0 $\pm$ 5870.2} & \textbf{\textcolor{red}{113658.2} $\pm$ 78878.8} & 59536.6 $\pm$ 46818.5 & \textbf{59787.4 $\pm$ 23864.5} \\
    B2-FA & 41.0 $\pm$ 16.3 & 72602.4 $\pm$ 104523.1 & 6066.6 $\pm$ 8068.0 & 163.2 $\pm$ 39.8 \\
    B3-TR & 3414.2 $\pm$ 7527.6 & 108205.8 $\pm$ 85854.4 & \textbf{63343.4 $\pm$ 41844.5} & 60335.2 $\pm$ 37423.7 \\
    B3-FA & 45.2 $\pm$ 11.0 & 71714.0 $\pm$ 103854.7 & 24663.2 $\pm$ 20909.8 & 167.0 $\pm$ 37.9 \\
    \bottomrule
  \end{tabular}
\end{table*}

\subsection{Antimicrobial Peptide (AMP) Discovery}

We evaluate BGFNs on a real combinatorial design task: the generation of antimicrobial peptides (AMPs) \cite{trabucco2022designbenchbenchmarksdatadrivenoffline, pirtskhalava_dbaasp}. 
Unlike prior work that fixes sequence length, here peptides have \emph{variable length} from 1 to 10 amino acids. 
We train a Random Forest proxy to predict the probability of antimicrobial activity and define the GFlowNet reward as the \emph{logit} of this probability. 
To avoid overweighting a small set of very-high-scoring sequences, we apply a 94\% cutoff and clip the log-reward at zero for sequences above this threshold (i.e., all peptides with predicted activity $\geq 0.94$ share the same terminal reward). 
This choice preserves ranking signal below the threshold while mitigating oversampling to a few modes above it.

\paragraph{Training protocol (AMPs).}
All models are trained for 3{,}000 epochs. 
We activate \textsc{Booster~1} at epoch 1{,}200 and \textsc{Booster~2} at epoch 2{,}400, keeping earlier stages frozen thereafter. 
We evaluate both on-policy and off-policy training under estimator noise levels $\varepsilon \in \{0.1, 0.2, 0.3\}$. 
In addition, we compare two loss instantiations for incorporating past rewards: the \emph{flow-additive} and the \emph{target-residual} variants. 
As in the synthetic setup, we reuse the same random seeds across conditions to ensure that observed differences arise from the training regime rather than luck.

\subsubsection{Results: BGFNs Discover Substantially More Unique High-Reward Peptides}

\begin{itemize}
    \item \textbf{BGFNs find more unique peptides.} In all experimental conditions, BGFNs discovered substantially more unique peptide sequences than standard GFlowNets. This advantage held consistently across all noise levels and for both on-policy and off-policy training. As shown in Figure~\ref{fig:uniq-peptides} and Table~\ref{tab:uniq-summary}, this amounts to orders of magnitude more high-confidence sequences.

    \item \textbf{Target-Residual (TR) excels at exploration.} Among the two BGFN instantiations, the Target-Residual (TR) variant achieved broader exploration and discovered more unique peptides. We attribute this to its more aggressive loss gradients that push probability mass into under-covered regions, which is beneficial when rewards are sparse.
\end{itemize}

\section{Conclusion, Limitations, and Broader Impact}

\paragraph{Conclusion.}
We introduced Boosted GFlowNets (BGFNs), a framework for sequentially training ensembles of GFlowNets. 
Our results demonstrate that BGFNs are \emph{safe to use}: new boosters can be added to improve exploration without degrading the performance of previously trained components. 
This property provides a principled way of enhancing exploration in challenging environments, where the topology of the state space creates strong biases towards easy-to-reach modes, even under off-policy training. 
Across synthetic multimodal benchmarks and antimicrobial peptide discovery, BGFNs consistently improved mode coverage and robustness, validating the framework as a reliable extension of standard GFlowNet training.

\paragraph{Limitations.}
The main limitation of BGFNs is computational: each additional booster requires sampling from the backwards policy of all previously trained models in the ensemble. 
This increases training time on the order of $\mathcal{O}(KN)$, where $K$ denotes the number of trajectories sampled for a terminal node and $N$ the number of GFlowNets in the ensemble. 
An important exception is \emph{sequence} environments with a unique backward path, in which the forward pass specifies the single valid backward trajectory; in practice, this allows reusing the same backward path across ensemble members rather than resampling for each model. 
In our experiments, we also mitigated cost by sampling only once per terminal state, but scaling to very large ensembles may still require additional engineering or approximations.

\paragraph{Broader Impact.}
BGFNs provide a general mechanism for continued training and integration of multiple GFlowNets, opening the door to ensembles that combine different architectures, training regimes, or even loss functions. 
Although we did not explore these extensions in this work, the framework naturally supports them. 
This flexibility suggests that BGFNs could serve as a unifying tool for generative modeling in diverse domains, ranging from scientific discovery tasks to applications in reinforcement learning and probabilistic inference, provided appropriate care is taken in designing the ensemble and its reward structure.


\section*{Acknowledgements}

We acknowledge the support by the Fundação Carlos Chagas Filho de Amparo à Pesquisa do Estado do Rio de Janeiro (FAPERJ) (SEI-260003/020348/2025, SEI-260003/020694/2025) and the Conselho Nacional de Desenvolvimento Científico e Tecnológico (CNPq) (404336/2023-0, 305692/2025-9). 

\bibliography{main}

\section*{Checklist}



\begin{enumerate}

  \item For all models and algorithms presented, check if you include:
  \begin{enumerate}
    \item A clear description of the mathematical setting, assumptions, algorithm, and/or model. [Yes] – Section 3 defines the formal objectives, balance conditions, and algorithms.
    \item An analysis of the properties and complexity (time, space, sample size) of any algorithm. [Yes] – The Limitations section discusses computational cost $\mathcal{O}(KN)$ of the ensemble.
    \item (Optional) Anonymized source code, with specification of all dependencies, including external libraries. [Yes] – Code and instructions will be provided in the supplementary material.
  \end{enumerate}

  \item For any theoretical claim, check if you include:
  \begin{enumerate}
    \item Statements of the full set of assumptions of all theoretical results. [Yes] – Theorems explicitly state assumptions on support and continuity.
    \item Complete proofs of all theoretical results. [Yes] – Proofs are included in Appendix~B.
    \item Clear explanations of any assumptions. [Yes] – Assumptions are explained in theorems and in the background section.
  \end{enumerate}

  \item For all figures and tables that present empirical results, check if you include:
  \begin{enumerate}
    \item The code, data, and instructions needed to reproduce the main experimental results (either in the supplemental material or as a URL). [Yes] – Code, data, and instructions will be released in the supplementary material.
    \item All the training details (e.g., data splits, hyperparameters, how they were chosen). [Yes] – Training protocols, epochs, seeds, and loss variants are described in Section 5 and Appendix.
    \item A clear definition of the specific measure or statistics and error bars (e.g., with respect to the random seed after running experiments multiple times). [Yes] – Results are reported as mean ± standard deviation over 5 seeds.
    \item A description of the computing infrastructure used. (e.g., type of GPUs, internal cluster, or cloud provider). [Yes] – All experiments were run on CPUs only.
  \end{enumerate}

  \item If you are using existing assets (e.g., code, data, models) or curating/releasing new assets, check if you include:
  \begin{enumerate}
    \item Citations of the creator If your work uses existing assets. [Yes] – The AMP dataset and proxy model are cited.
    \item The license information of the assets, if applicable. [Yes] – The AMP dataset is publicly available; license details will be clarified in the supplementary material.
    \item New assets either in the supplemental material or as a URL, if applicable. [Not Applicable]
    \item Information about consent from data providers/curators. [Not Applicable]
    \item Discussion of sensible content if applicable, e.g., personally identifiable information or offensive content. [Not Applicable]
  \end{enumerate}

  \item If you used crowdsourcing or conducted research with human subjects, check if you include:
  \begin{enumerate}
    \item The full text of instructions given to participants and screenshots. [Not Applicable]
    \item Descriptions of potential participant risks, with links to Institutional Review Board (IRB) approvals if applicable. [Not Applicable]
    \item The estimated hourly wage paid to participants and the total amount spent on participant compensation. [Not Applicable]
  \end{enumerate}

\end{enumerate}

\clearpage

\onecolumn

\thispagestyle{empty}
\pagestyle{empty}

\aistatstitle{Boosted GFlowNets: Improving Exploration via Sequential Learning\\
Supplementary Materials}

\appendix
\section{Notation and objectives}

\paragraph{Trajectories and terminal states.}
A trajectory is denoted by \(\tau\). For each terminal state \(x\), let \(\Gamma_x\) be the set of all trajectories that terminate at \(x\).
Throughout, when an expectation or integrand depends on \(\tau\), the symbol \(x\) denotes the terminal state of that \(\tau\) (i.e., \(\tau\in\Gamma_x\)).

\paragraph{Forward and backward laws.}
We write \(P_F(\tau)\) for the probability of \(\tau\) under the forward sampler (policy and environment dynamics), and
\(P_B(\tau\mid x)\) for the probability of \(\tau\) under the backward sampler \emph{conditioned} on the terminal \(x\).
All statements below are made for \(\tau\in\Gamma_x\).
When using frozen ensemble members (indexed by \(k\)), we denote their normalizers by \(Z_k\) and their
forward/backward trajectory laws by \(P_F^{(k)}(\tau)\) and \(P_B^{(k)}(\tau\mid x)\).

\paragraph{Rewards and normalizer.}
Each terminal \(x\) has a strictly positive, unnormalized reward \(R(x)>0\); we work in log-space via \(\log R(x)\).
The scalar \(Z_\theta>0\) denotes the (learned) normalizer for the current model parameters \(\theta\).

\paragraph{Trajectory-level estimator (single member).}
Given \(\tau\in\Gamma_x\), the trajectory-balance estimator of the terminal reward is
\begin{equation}
\label{eq:s1-rhat}
\widehat{R}_\theta(x;\tau)\;:=\; Z_\theta\,\frac{P_F(\tau)}{P_B(\tau\mid x)}.
\end{equation}

\paragraph{Trajectory-balance (TB) objective.}
The TB loss for a single trajectory \(\tau\in\Gamma_x\) is the squared log discrepancy
\begin{equation}
\label{eq:s1-tb}
\mathcal{L}_{\mathrm{TB}}(\tau)\;=\;\Big(\,\log \widehat{R}_\theta(x;\tau)\;-\;\log R(x)\,\Big)^2.
\end{equation}
Training minimizes the empirical expectation of \(\mathcal{L}_{\mathrm{TB}}\) over trajectories sampled from the forward process.

\paragraph{\(\alpha\)-Boosted objective}\label{sec:s1-boosting}
Let \({R}_\theta(x;\tau)\) be the current trainable estimator.
For a frozen ensemble \(\mathcal{F}\), we define the induced reward terminal estimate as
\[
\widehat R(x)\;:=\;\sum_{k\in\mathcal{F}}\;\E_{\tau\sim P_B^{(k)}(\cdot\mid x)}\!\Big[Z_k \frac{P_F^{(k)}(\tau)}{P_B^{(k)}(\tau \mid x)}\Big],
\]
Given \(\alpha\in[0,1]\), we define the boosted loss
\begin{equation}
\label{eq:s1-boost}
\mathcal{L}_{\mathrm{boost}}^{(k)}(\tau)
\;=\;
\left(
\log\frac{R_\theta(x;\tau)\;+\;\alpha\,\widehat R_k(x)}{\,R(x)\;-\;(1-\alpha)\,\widehat R_k(x)\,}
\right)^2.
\end{equation}

Here, $\widehat{R}_k(x)$ is a $k$-sample Monte Carlo estimate of the frozen model's induced terminal estimate $\widehat{R}(x)$. To simplify notation we omit k and assume only one Monte Carlo sample from each model in the ensemble. 

\emph{Positivity of the denominator.}
To ensure \(R(x)-(1-\alpha) \widehat R(x)>0\) uniformly, we replace \(\alpha\) by a per-terminal state clamped value
\[
\alpha_t(x)\;=\;\mathrm{clip}\!\Big(\alpha,\ \alpha_{\min}(x),\ 1\Big),
\qquad
\alpha_{\min}(x)\;=\;
\begin{cases}
0, & \widehat R(x)=0,\\[4pt]
1-\dfrac{R(x)-\delta}{\widehat R(x)}, & \widehat R(x)>0,
\end{cases}
\]
with tiny \(\delta>0\) (on the order of machine epsilon). Using \(\alpha_t(x)\) in \eqref{eq:s1-boost} guarantees
\[
R(x)\;-\;(1-\alpha_t(x))\,\widehat R(x)\ \ge\ \delta\ >0
\]

\emph{Special cases.} If \(\widehat R\equiv 0\), then \(\mathcal{L}_{\mathrm{boost}}(\tau)=(\log R_\theta-\log R)^2\), i.e., it reduces to TB.
If \(\widehat R(x)=R(x)\), then \(\mathcal{L}_{\mathrm{boost}}(\tau)=\big(\log(R_\theta/(\alpha R)+1)\big)^2\).
(See Section~\ref{sec:proofs} for the accompanying statements and proofs.)

\paragraph{Expectations and variance.}

We use
\(\mathbb{E}_{\tau\sim P_F}[\cdot]\) for forward-sampled expectations,
\(\mathbb{E}_{\tau\sim P_B(\cdot\mid x)}[\cdot]\) for backward-sampled expectations conditional on \(x\), and
\(\operatorname{Var}[\cdot]\) for the corresponding variances.
All equalities and inequalities are understood to hold almost surely with respect to the relevant sampling law.

\paragraph{Estimating \texorpdfstring{$\widehat{R}$}{Rhat} via importance sampling}
\label{sec:estimator}

We define the implicit reward (the terminal \emph{flow} in standard GFlowNet terminology) as
\begin{equation}
\label{eq:rhat-sum}
\widehat{R}(x)
\;\stackrel{\mathrm{def}}{=}\;
Z \sum_{\tau:\,\tau\to x} P_F(\tau)
=Z\,P_F(x).
\end{equation}
Without enumerating paths, rewrite it as an expectation under any reverse-path distribution $P_B(\cdot\mid x)$ supported on $\{\tau:\tau\to x\}$:
\begin{equation}
\label{eq:rhat-IS}
\widehat{R}(x)
=
\sum_{\tau:\,\tau\to x} Z P_F(\tau)\frac{P_B(\tau \mid x)}{P_B(\tau \mid x)}
=
\E_{\tau\follows P_B(\cdot\mid x)}\!\Big[\tfrac{Z\,P_F(\tau)}{P_B(\tau\mid x)}\Big].
\end{equation}

\section{Grid Environment:}\label{sec:grid}

\paragraph{State, time, and domain.}
We use a square, symmetric grid with explicit time such that every possible state is reachable:
\[
\mathcal{X}=\{-W,\ldots,W\}\times\{-W,\ldots,W\},\qquad
s=(x,y,t),\quad t\in\{0,\ldots,T\},\quad T=2W.
\]
Where the policy observes $\big(x,\;y,\;t/T\big)$.

\paragraph{Actions and dynamics.}
The action set is
\[
\mathcal{A}=\{\rightarrow,\leftarrow,\uparrow,\downarrow,\varnothing\}
=\{(1,0),(-1,0),(0,1),(0,-1),(0,0)\}.
\]
A \emph{forward} step (enabled if \(t<T\)) updates
\[
(x',y',t')=\big(x+\Delta x,\;y+\Delta y,\;t+1\big).
\]
A \emph{backward} step (enabled if \(t>0\)) updates
\[
(x',y',t')=\big(x-\Delta x,\;y-\Delta y,\;t-1\big).
\]

\paragraph{Valid-action masks.}
Invalid logits are set to a large negative constant before softmax.
The \emph{forward mask} \(M_F(s,a)=1\) iff \(t<T\) and the next position is in bounds.
The \emph{backward mask} \(M_B(s,a)=1\) iff
(i) \(t>0\); (ii) the previous position is in bounds; and
(iii) the predecessor is reachable at time \(t-1\):
\[
|x-\Delta x|+|y-\Delta y|\;\le\;t-1 .
\]

\paragraph{Policies.}
We use separate neural policies for forward and backward. Each is a masked MLP over \([x,y,t/T]\) with geometric Fourier time features \((f_k=2^k)\), producing 5 logits for the actions; masking precedes the softmax.

\paragraph{Reward families (\texttt{8g}, \texttt{rings}, \texttt{moons}).}
Rewards depend only on position \((x,y)\). We define a density \(\rho(x,y)\ge 0\) and set
\[
\log R(x,y)=\log\!\big((1-\lambda)\,\rho(x,y)+\lambda\big),\qquad \lambda\ll 1,
\]
to avoid zeros.

\emph{Eight Gaussians (8g).} With anchors \(a_m=(R\cos\theta_m, R\sin\theta_m)\), \(\theta_m=2\pi m/8\), \(m=0,\dots,7\):
\[
\rho(x,y)=\sum_{m=0}^{7}\exp\!\left(-\tfrac{1}{2\sigma^2}\,\big\|(x,y)-a_m\big\|^2\right).
\]

\emph{Rings.} For radii \(\{r_\ell\}\) with width \(\sigma_r\) and optional weights \(w_\ell\):
\[
\rho(x,y)=\sum_{\ell} w_\ell\;\exp\!\left(-\tfrac{1}{2\sigma_r^2}\,\big(\|(x,y)\|-r_\ell\big)^2\right).
\]

\emph{Two-moons (moons).} Two semicircles of radius \(R\), centered at \((-\delta,0)\) (upper arc) and \((+\delta,-\mathrm{gap})\) (lower arc). Discretize each arc with anchors \(\{a_m\}\) and sum isotropic Gaussians of width \(\sigma\):
\[
\rho(x,y)=\sum_{m}\exp\!\left(-\tfrac{1}{2\sigma^2}\,\big\|(x,y)-a_m\big\|^2\right).
\]

\paragraph{Reward instantiation}
We instantiate each family by specifying a small set of shape parameters; the final log-reward is
\[
\log R(x,y)=\log\!\big((1-\lambda)\,\rho(x,y)+\lambda\big),\qquad \lambda=10^{-6},
\]
with \(W=H\). \cref{tab:s3-rewards} summarizes the defaults we use.

\begin{table}[h]
\begin{tabular}{ll}
\hline
Family & Defaults \\
\hline
8 Gaussians (8g) & radius \(R=0.8\,W\); Gaussian width \(\sigma=1.0\). \\
Rings            & radii \(r\in\{0.4\,W,\ 0.8\,W\}\); radial width \(\sigma_r=1\); weights \(w=(1,\,1)\). \\
Two moons        & radius \(R=0.6\,W\); offset \(\delta=0.03\,R\); vertical gap \(0.018\,R\); width \(\sigma=1\); anchors \(n=256\). \\
\hline
\end{tabular}
\centering
\caption{Default parameterization of grid rewards (square grid with half-width \(W\)).}
\label{tab:s3-rewards}
\end{table}

\paragraph{Sampling.}
\emph{Forward} uses \(\varepsilon\)-mixing with the uniform distribution over valid actions:
\[
\tilde{\pi}_F(a\mid s)=(1-\varepsilon)\,\pi_F(a\mid s)+\varepsilon\cdot \mathcal{U}\big(\{a:\,M_F(s,a)=1\}\big),
\]
and \(\varepsilon=0\) at evaluation. \emph{Backward} is categorical from the masked \(\pi_B(\cdot\mid s)\).

\paragraph{Training setup and hyperparameters.}
We train with AdamW on batches of on-policy forward rollouts.
\[
\mathcal{L}_{\mathrm{TB}}=\mathbb{E}\Big[\big(\log \widehat{R}_\theta(x;\tau)-\log R(x)\big)^2\Big].
\]
\emph{Baseline} (\(K{=}1\)): single forward/backward pair \((\pi_F,\pi_B)\) and a scalar \(\log Z\).
\emph{Boosting}: freeze previously trained members and add a new forward/backward pair trained with the \(\alpha\)-Boosted objective (Sec.~\ref{sec:s1-boosting}).

\paragraph{Training configuration.}
\begin{table*}[h]
\centering
\caption{Default training setup for the grid environment.}
\label{tab:s3-train}
\begin{tabular}{@{}l p{0.72\linewidth}@{}}
\hline
Component & Setting \\
\hline
Grid \& horizon   & $W=15$, $T=2W$. \\
Terminal states   & $961$. \\
Batch size        & 128. \\
Exploration       & $\varepsilon\in\{0,\,0.1,\,0.2,\,0.3,\,0.4,\,0.5\}$ (evaluation uses $\varepsilon=0$). \\
Seeds             & $\{10,\ldots,15\}$. \\
Optimizer         & AdamW; learning rates $(\mathrm{pf},\,\mathrm{pb},\,\log Z)=(10^{-2},\,10^{-2},\,5\times 10^{-2})$. \\
Policy nets       & Forward/backward MLP with geometric Fourier time features; hidden size 128; \\
                  & 2 layers; 8 frequencies; input $[x,y,t/T]$. \\
Schedule          & Baseline: $10{,}000$ epochs; Booster \#1: resume from checkpoint $3{,}000$ ($\alpha=1.0$); \\
                  & Booster \#2: resume from checkpoint $6{,}000$ ($\alpha=1.0$). \\
\hline
\end{tabular}
\end{table*}

\noindent\emph{Seed control.} To prevent subsequent boosters from exploring new areas by coincidence (rather than due to the ensemble’s intrinsic reward), we initialize \textbf{all} boosters with the same seeds as the baseline (same seed for forward/backward policies and for the environment RNGs).

\paragraph{Evaluation metric.}
On the full terminal slice \(t=T\) we compare the model to the normalized reward:
\[
p^\star(x)=\frac{R(x)}{\sum_{x'} R(x')},\qquad
p(x)= \frac{\widehat{R}(x)}{\widehat{Z}},
\]
with \(\widehat{Z}=\sum_{x'}\widehat{R}(x')\).
In practice, we estimate \(\widehat{R}(x)\) by Monte Carlo marginalization over backward paths \emph{per ensemble member}:
for each terminal \(x\) and each member \(k\), we draw \(B=10\) backward trajectories
\(\tau^{(b)}\sim P_B^{(k)}(\cdot\mid x)\) and set
\[
\widehat{R}^k(x)\;:=\;\frac{1}{B}\sum_{b=1}^{B}
\frac{Z_k\,P_F^{(k)}(\tau^{(b)})}{P_B^{(k)}(\tau^{(b)}\mid x)}
\]
then aggregate \(\widehat{R}(x)=\sum_k \widehat{R}^k(x)\).
We report
\[
\mathrm{L}_1 \;=\; \frac{1}{|\mathcal{X}|}\sum_{x\in\mathcal{X}}\big|\,p^\star(x)-p(x)\,\big|.
\]

\begin{figure*}[h]
  \centering
  \setlength{\tabcolsep}{2pt}            
  \renewcommand{\arraystretch}{0}         
  \begin{tabular}{*{6}{c}}
    \includegraphics[width=0.155\textwidth]{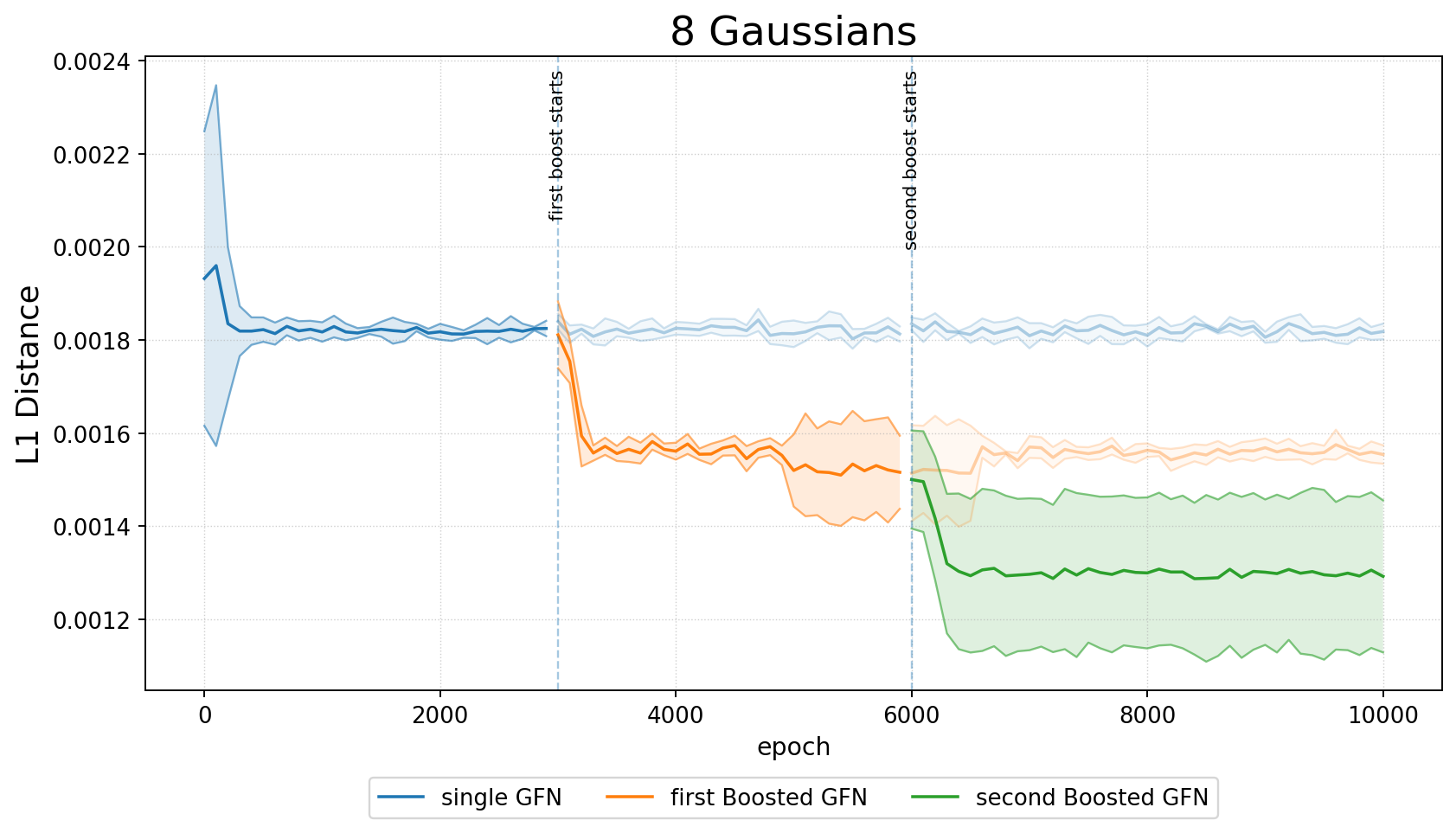} &
    \includegraphics[width=0.155\textwidth]{figs/8g/eps-0.1.png} &
    \includegraphics[width=0.155\textwidth]{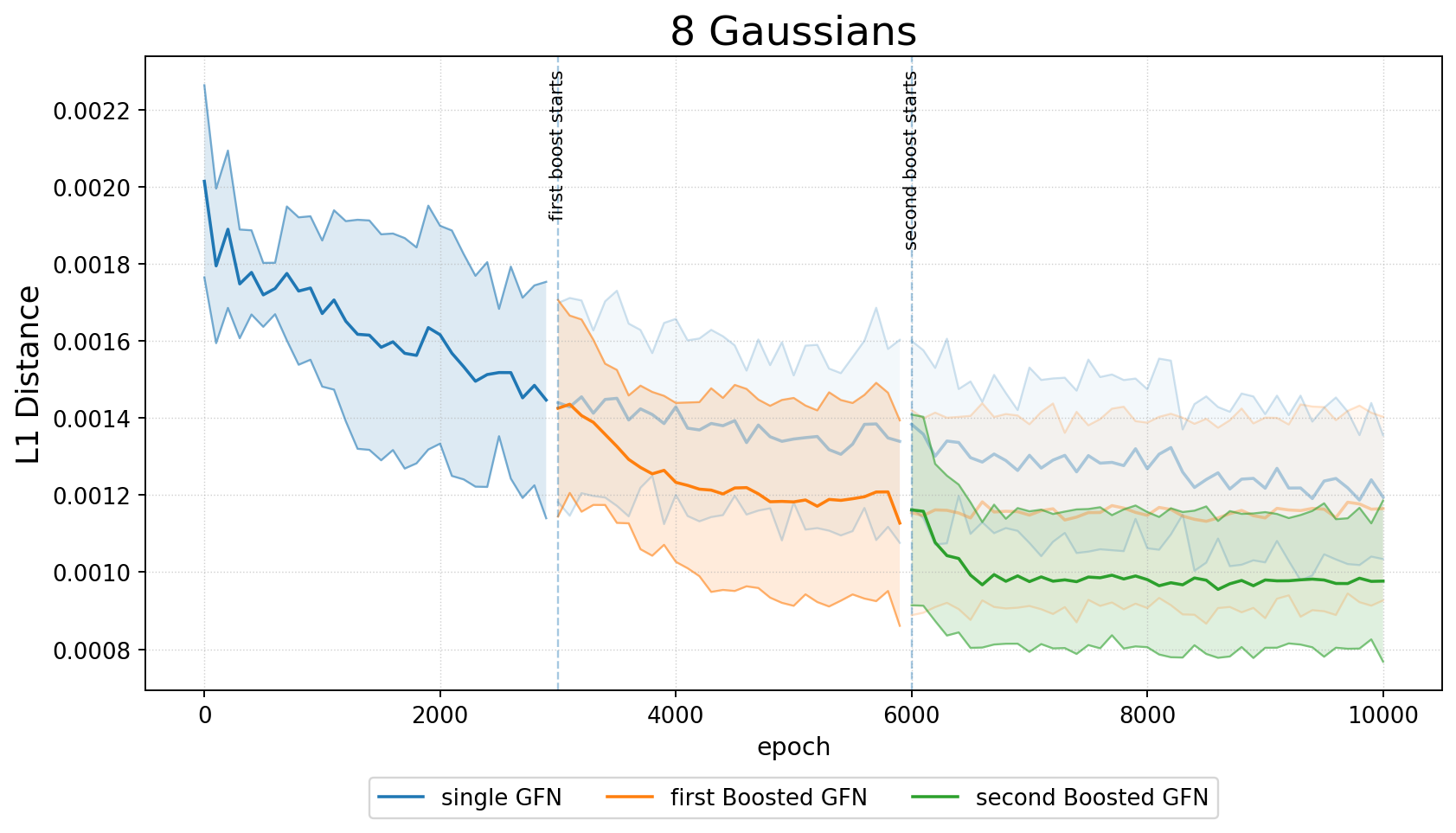} &
    \includegraphics[width=0.155\textwidth]{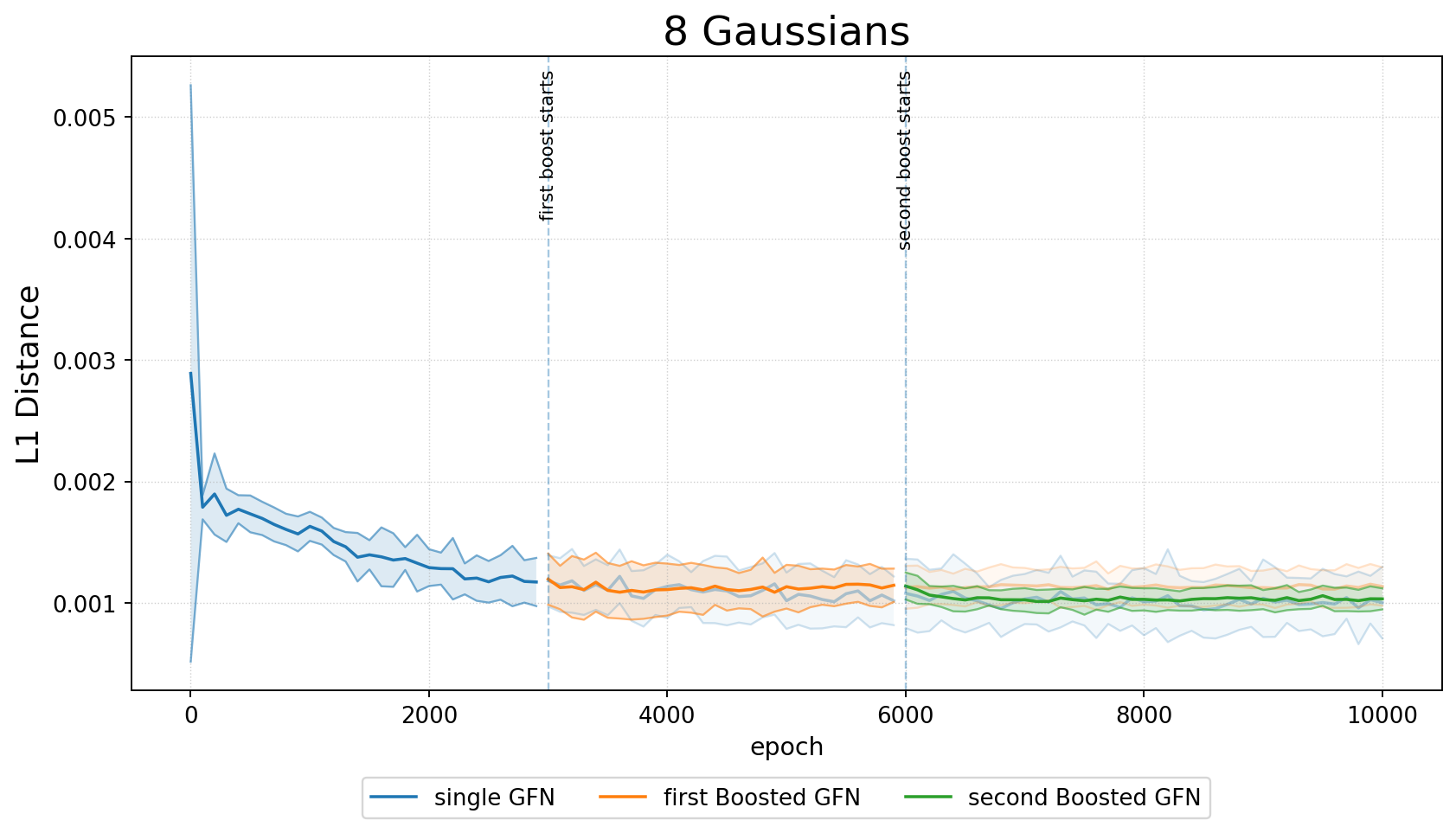} &
    \includegraphics[width=0.155\textwidth]{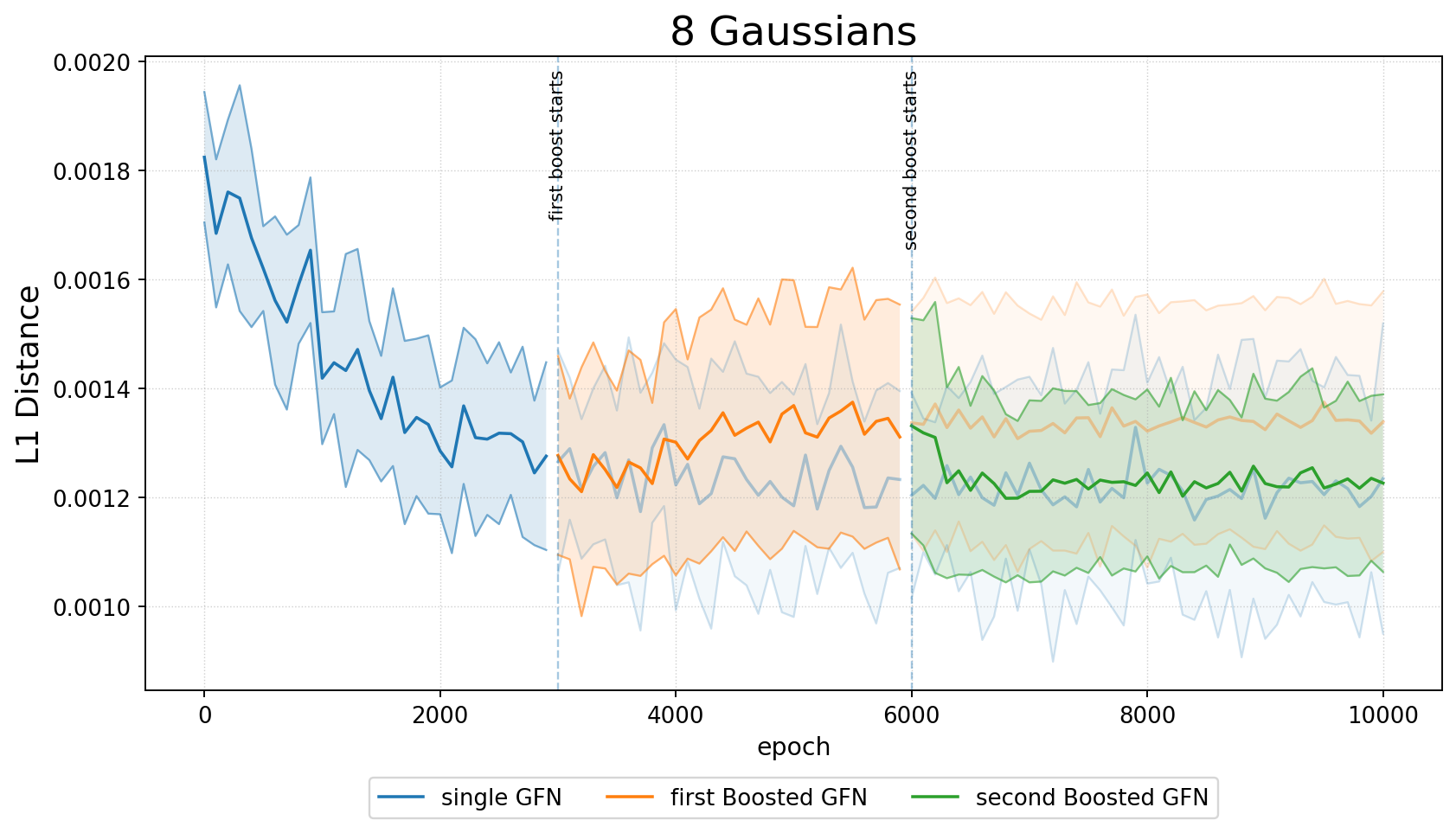} &
    \includegraphics[width=0.155\textwidth]{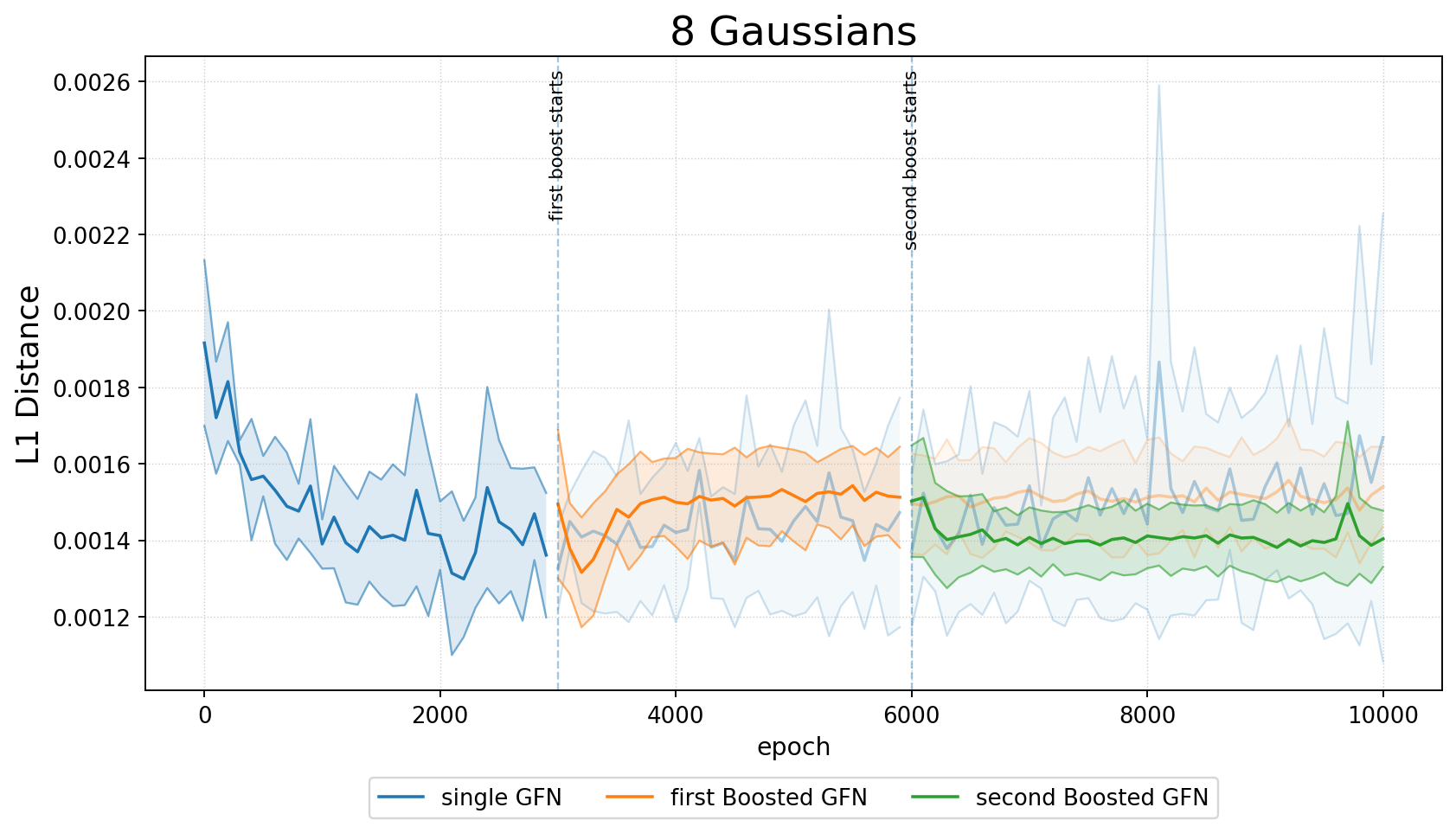} \\

    \includegraphics[width=0.155\textwidth]{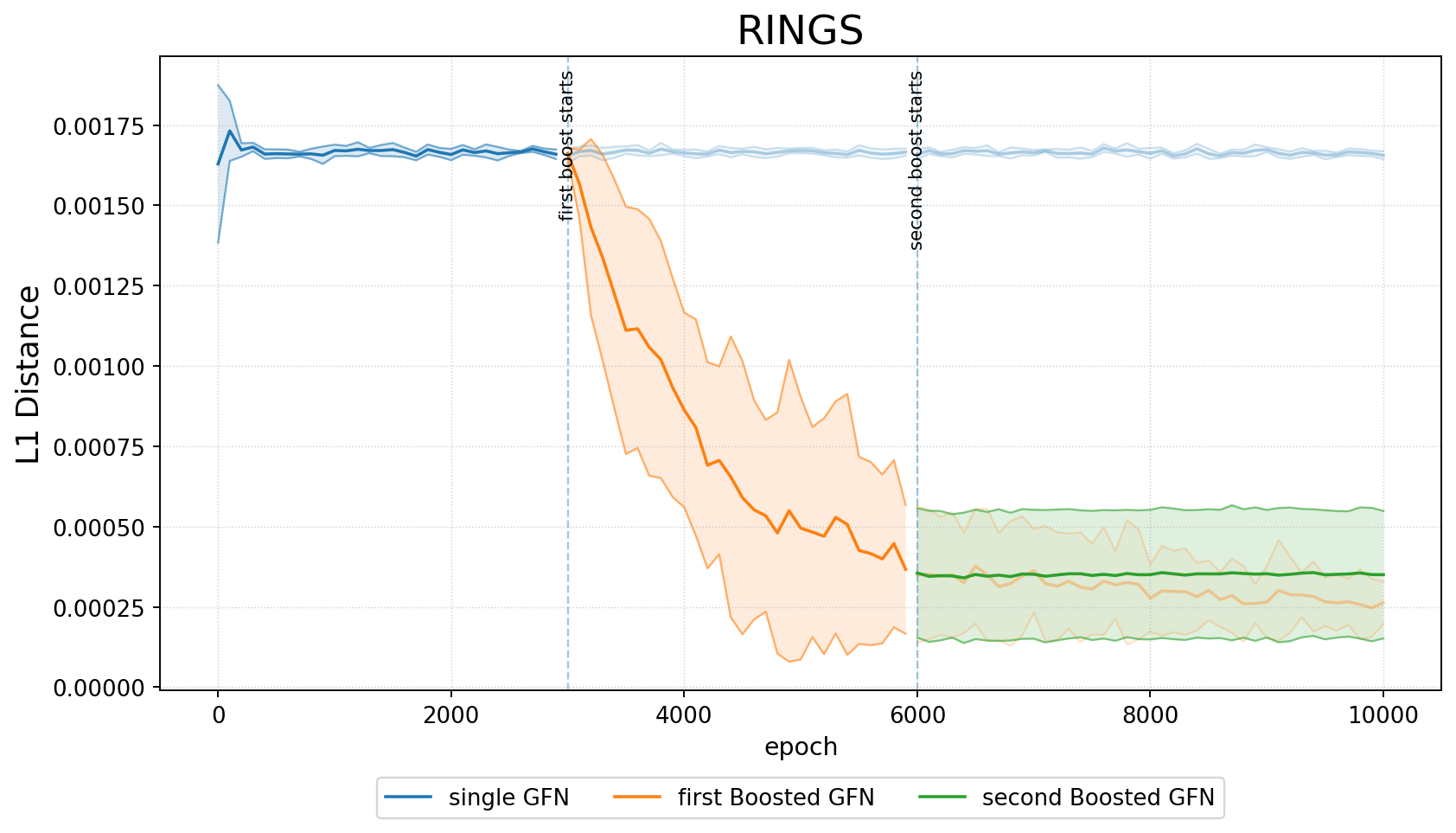} &
    \includegraphics[width=0.155\textwidth]{figs/rings/eps-0.1.png} &
    \includegraphics[width=0.155\textwidth]{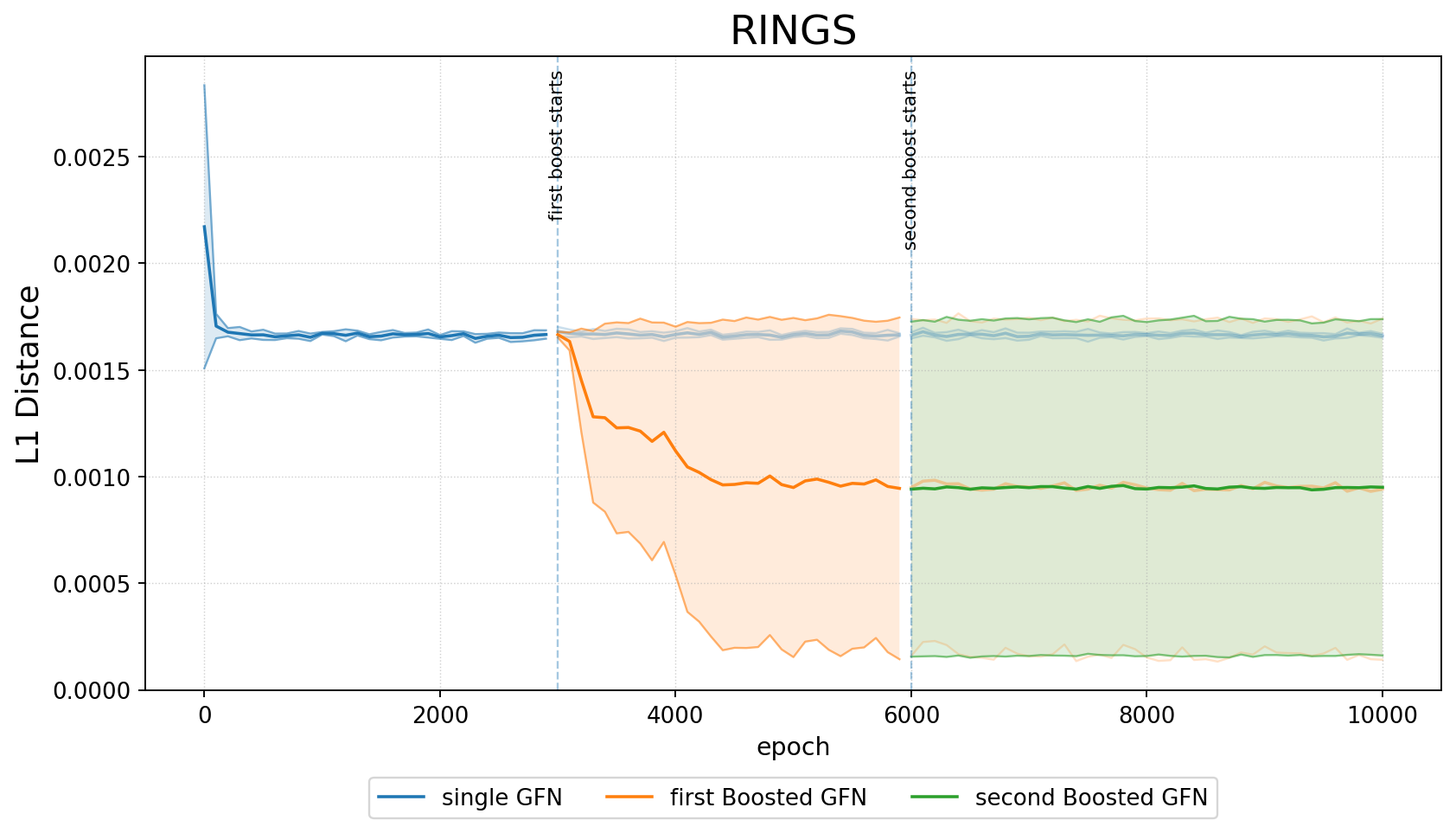} &
    \includegraphics[width=0.155\textwidth]{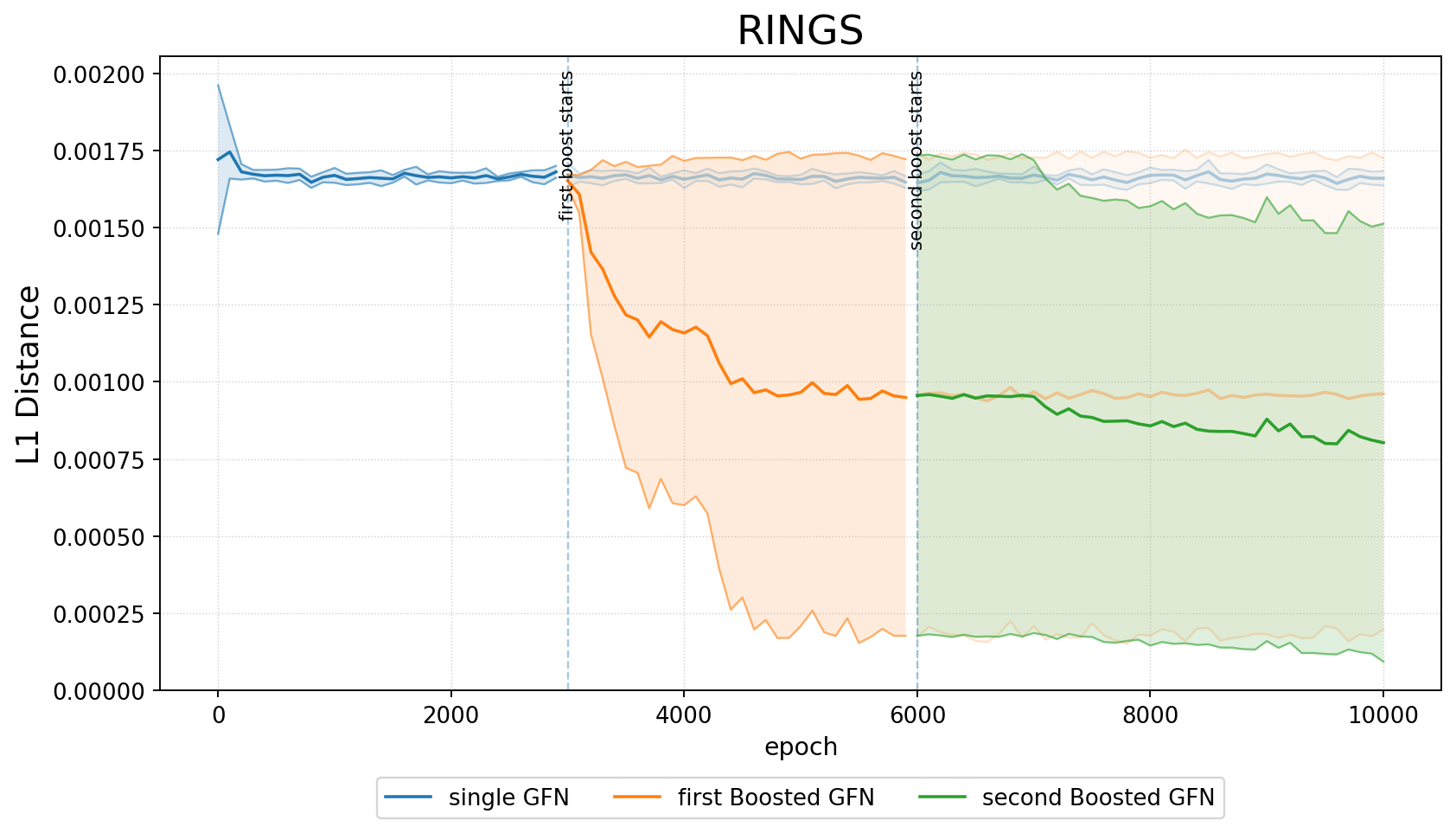} &
    \includegraphics[width=0.155\textwidth]{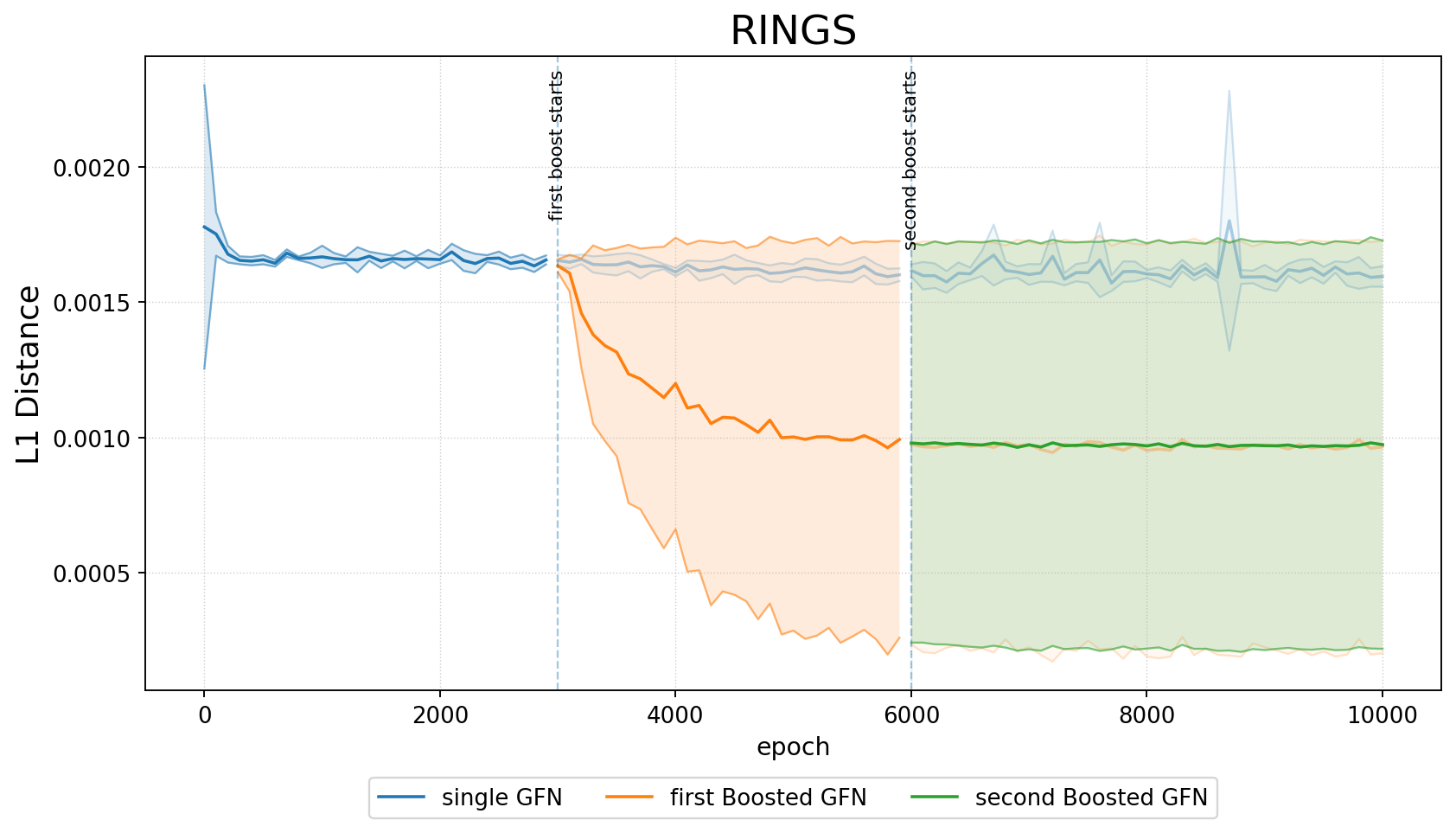} &
    \includegraphics[width=0.155\textwidth]{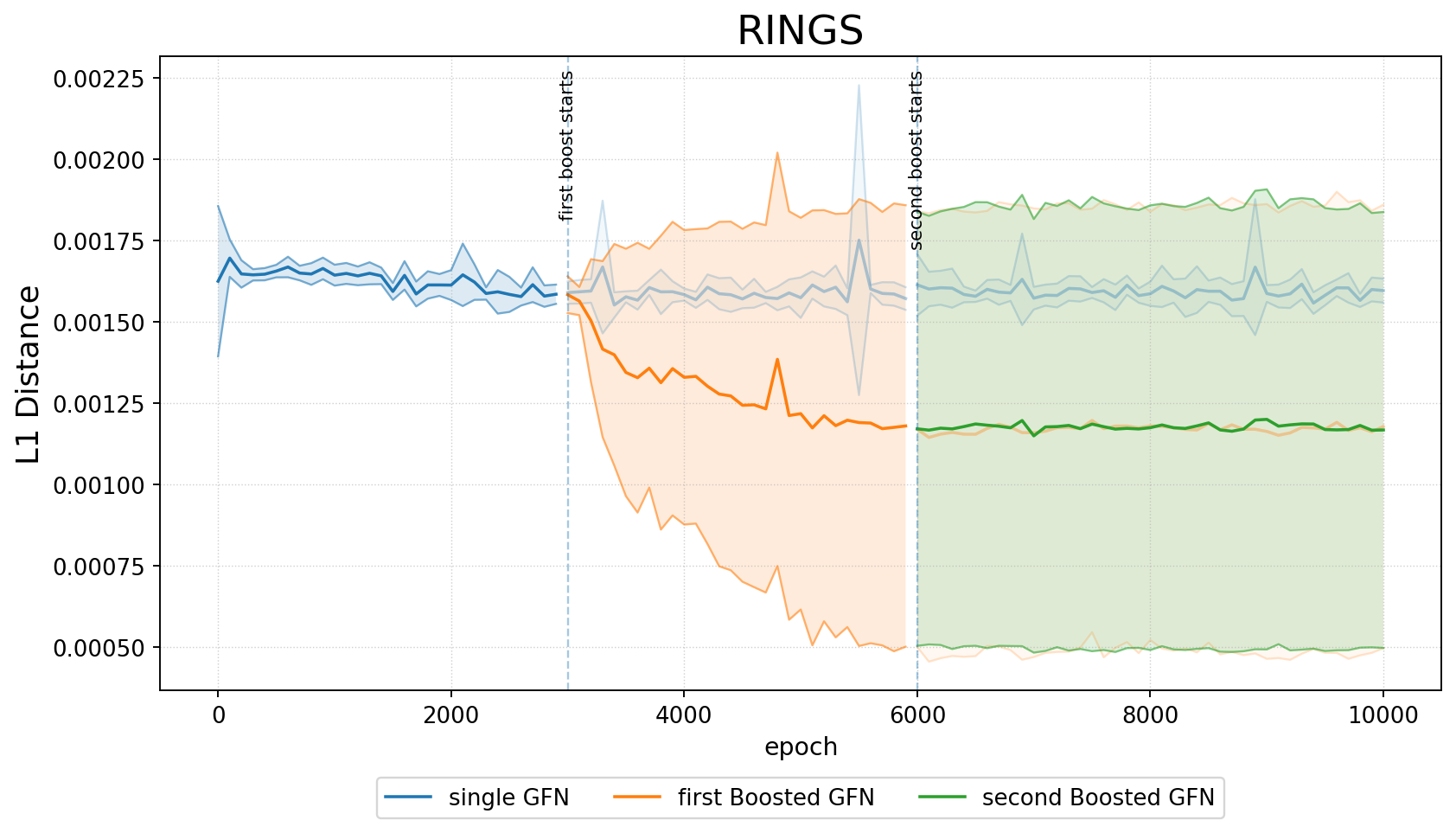} \\

    \includegraphics[width=0.155\textwidth]{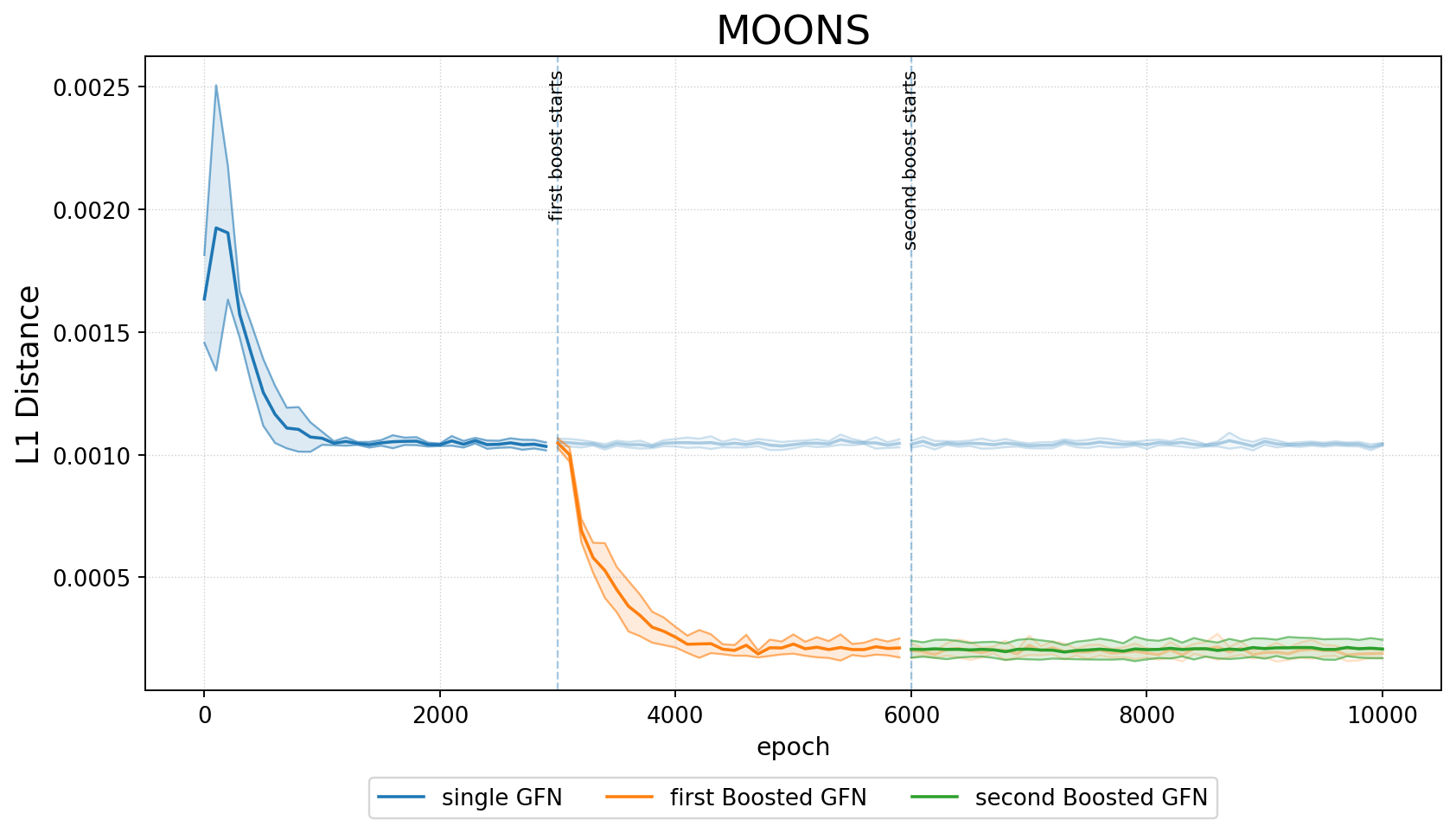} &
    \includegraphics[width=0.155\textwidth]{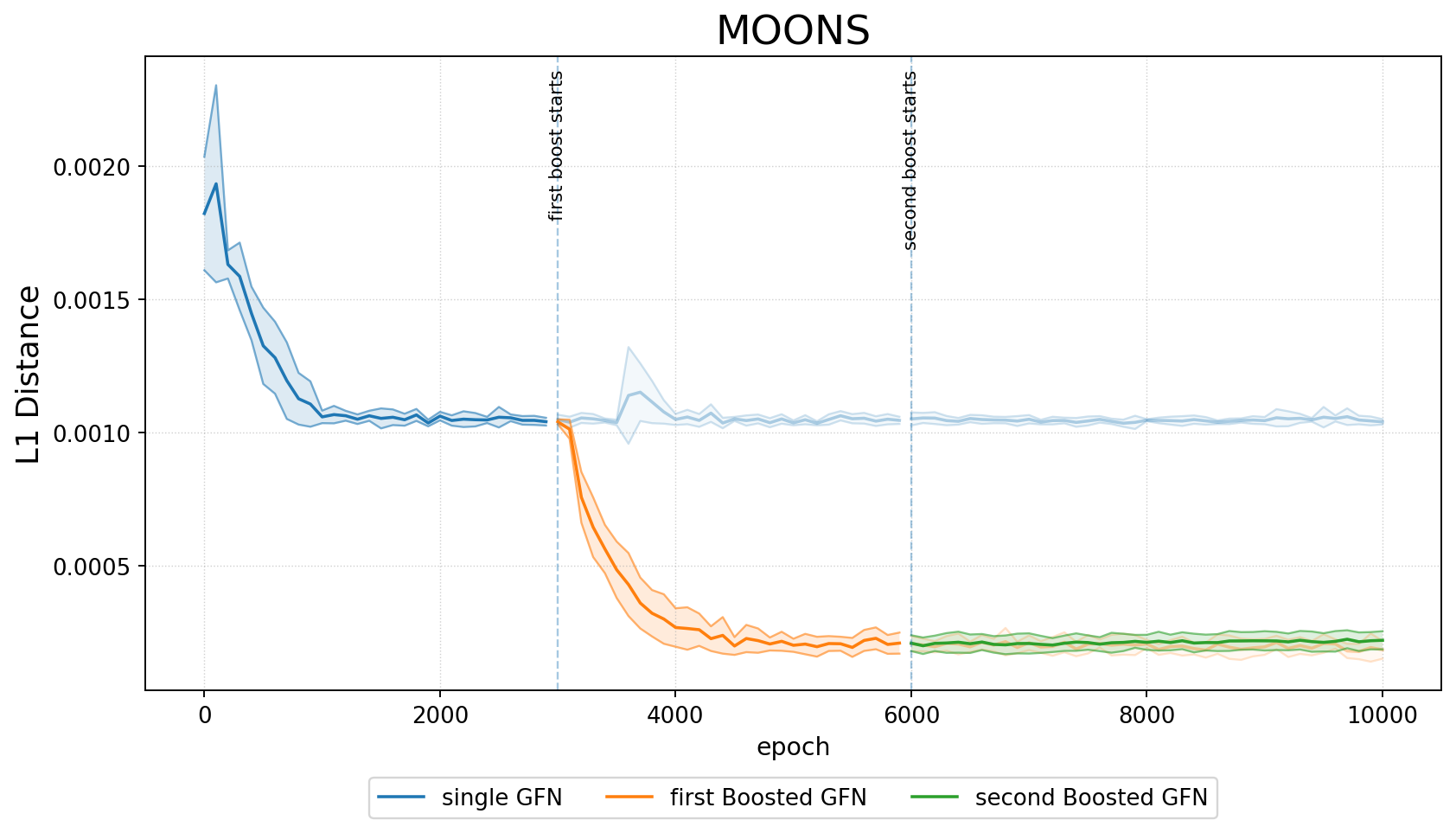} &
    \includegraphics[width=0.155\textwidth]{figs/moons/eps-0.2.png} &
    \includegraphics[width=0.155\textwidth]{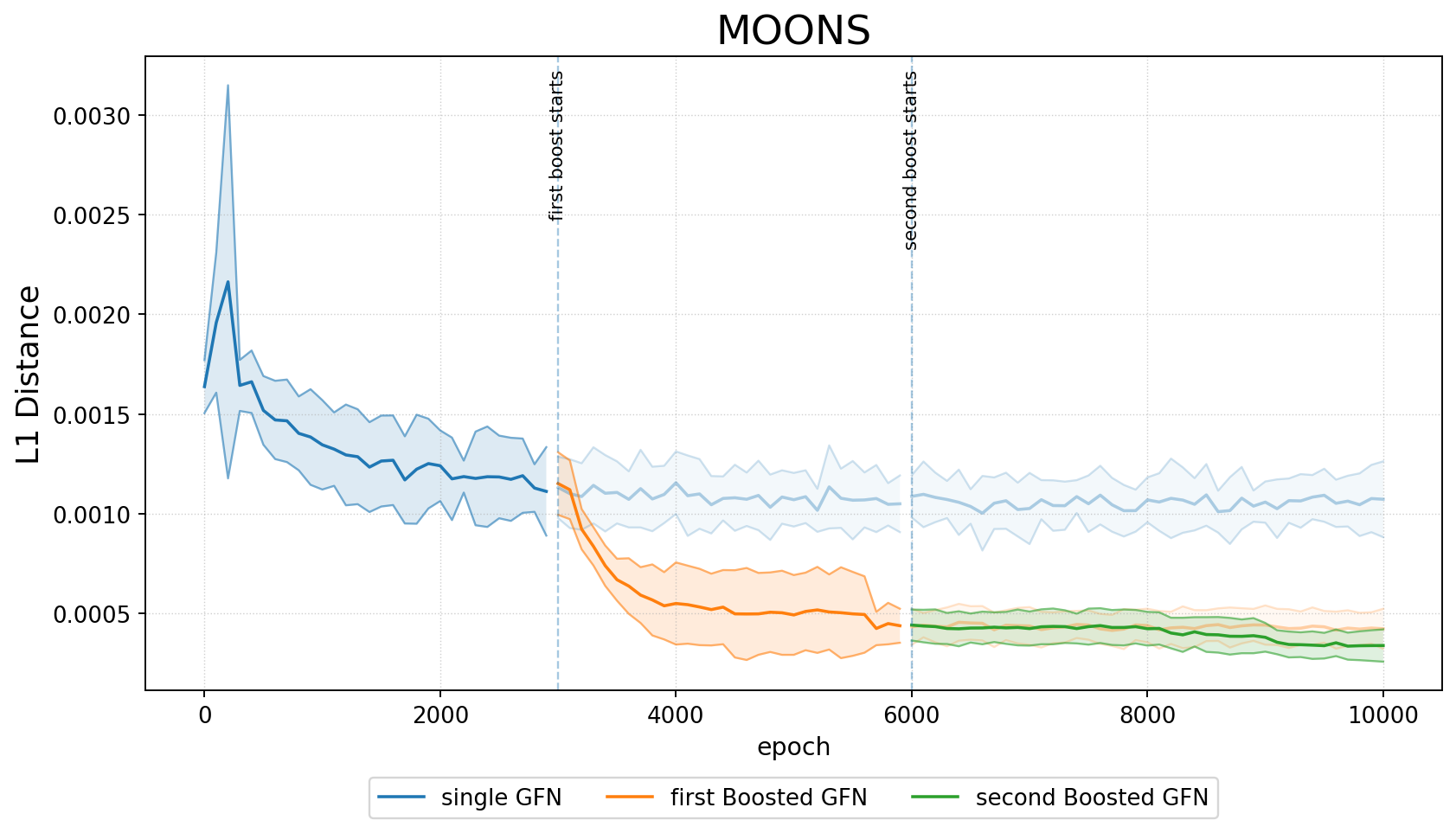} &
    \includegraphics[width=0.155\textwidth]{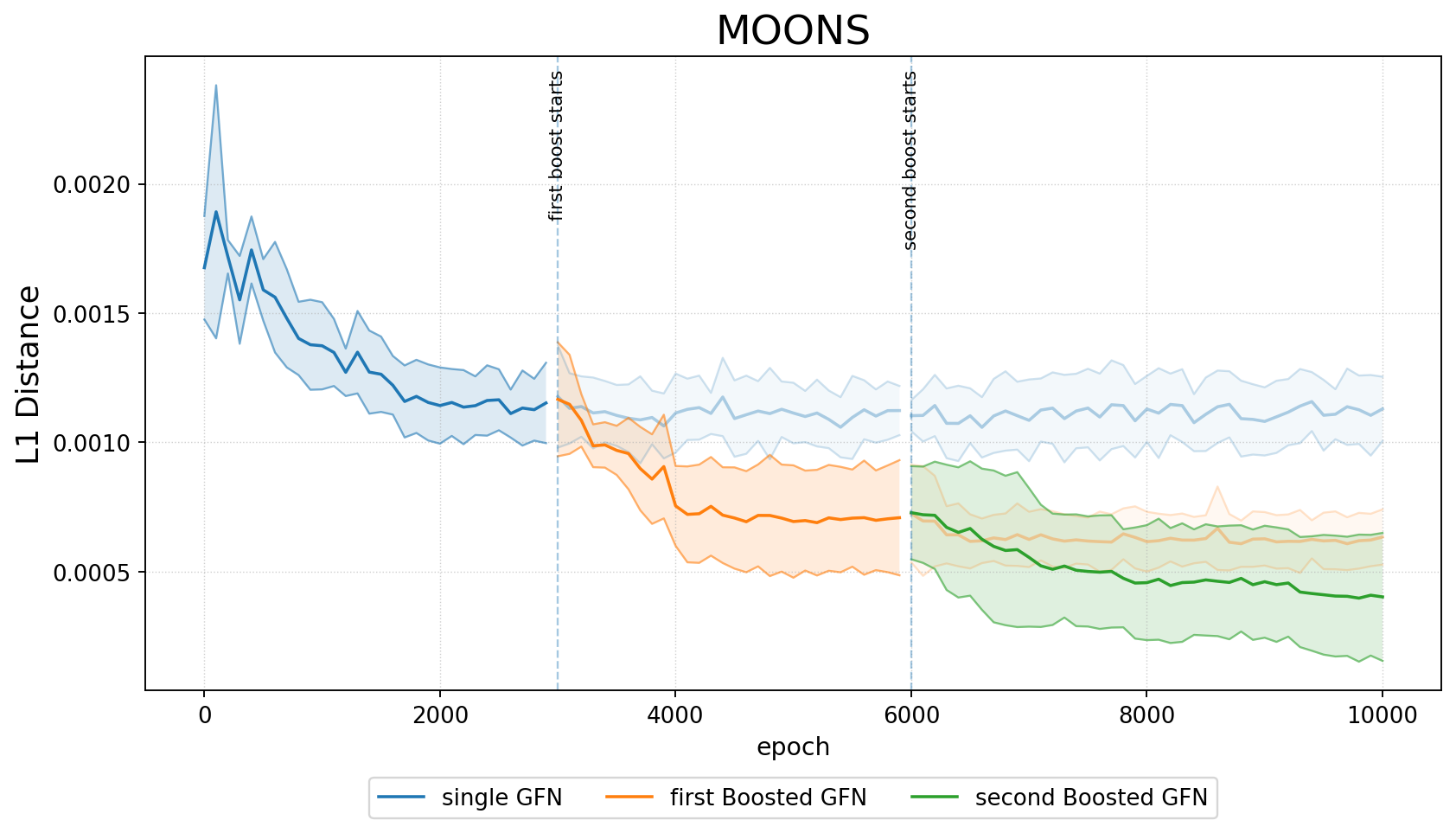} &
    \includegraphics[width=0.155\textwidth]{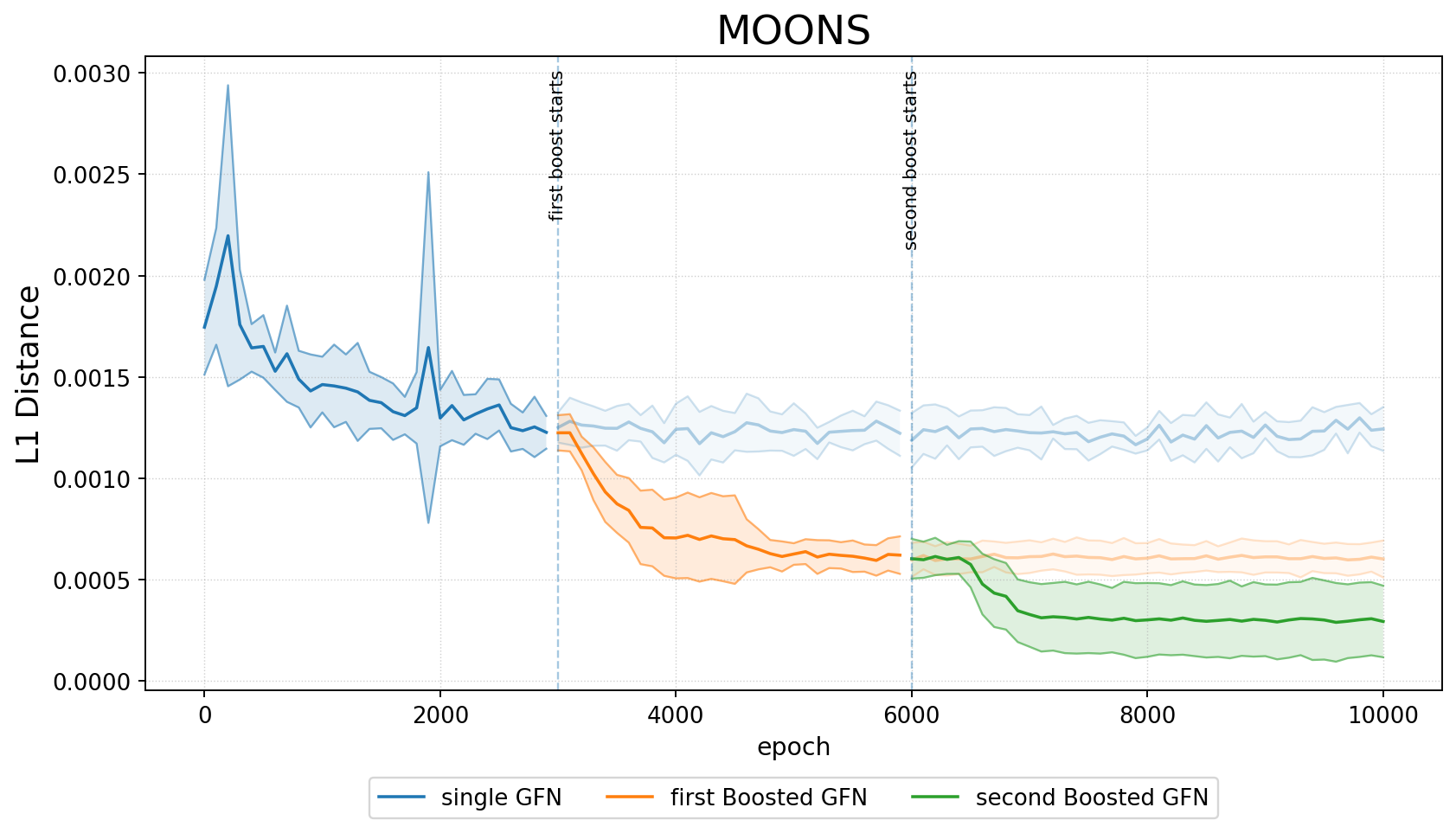}
  \end{tabular}

  \caption{\textbf{Synthetic targets across exploration levels.}
  Columns correspond to $\varepsilon\in\{0,0.1,0.2,0.3,0.4,0.5\}$.
  Rows (top to bottom): \textsc{Eight-Gaussians}, \textsc{Rings}, \textsc{Moons}.
  Compare L1 distance between the true probability and the model as exploration increases.}
  \label{fig:synth-eps-grid}
\end{figure*}

\section{Peptide Generation}\label{sec:peptides}

\paragraph{Sequence space and length.}
We generate variable-length peptides over an alphabet without Cysteine.
Let the vocabulary be
\[
\mathcal{V}=\{\texttt{STOP}\}\,\cup\,\{A,D,E,F,G,H,I,K,L,M,N,P,Q,R,S,T,V,W,Y\},
\quad |\mathcal{V}|=20,
\]

Sequences have length
\(L\in\{1,\dots,L_{\max}\}\) with \(L_{\max}=10\).
A trajectory is a token sequence \(y_{1:L}\in\mathcal{V}^L\).
The internal state is a fixed-width array \(Y\in\{0,\dots,19\}^{L_{\max}}\) (right-padded with
\(=\!0\)); the current step is \(t=\#\{i:\,Y_i\neq 0\}\).

\paragraph{Actions and dynamics.}
At step \(t\) the forward policy chooses \(a_t\in\{0,\dots,19\}\) (indexing \(\mathcal{V}\)):
\[
a_t=\begin{cases}
0 & \text{write \texttt{STOP} and terminate}\\
1..19 & \text{append the corresponding amino acid and continue.}
\end{cases}
\]
The environment writes \(Y_{t+1}\gets a_t\). If \(t=L_{\max}\) we force \(a_t=0\) (termination).

\emph{Backward actions.}
Given a nonempty prefix \(y_{1:t}\) with \(t\ge 1\), the only valid backward move is to undo the last forward action (pop the last token):
\[
\pi_B(a^\leftarrow \mid y_{1:t}) \;=\;
\begin{cases}
1, & a^\leftarrow = y_t \text{ and } t\ge 1,\\
0, & \text{otherwise,}
\end{cases}
\qquad
(y'_{1:t-1},\,t-1)\;=\;\mathrm{pop}(y_{1:t}).
\]

\paragraph{Policy architecture.}
We use an autoregressive policy \(\pi_\theta(a_t\mid y_{1:t})\) implemented by a light MLP at each step \(t\), the policy builds a feature vector
\(\phi_t \in \mathbb{R}^{W d_e + d_p}\) by (i) taking the last \(W=6\) tokens
of the prefix (indices clamped so that early steps use the padding token) and
embedding them with \(d_e=64\), then (ii) concatenating a sinusoidal positional
encoding of the current step of dimension \(d_p=16\).
With \(|\mathcal V|=20\) tokens where \texttt{EOS/STOP} shares index \(0\)
with the padding token, the embedding row for index \(0\) is fixed to zero.
The network is a single hidden-layer MLP
with hidden size \(128\).
When the buffer is full (\(t=L_{\max}\)), we hard-mask non-\texttt{EOS/STOP} logits
to \(-\infty\) to force termination.
During training we sample with \(\varepsilon\)-mixing against the uniform
distribution over the \(20\) actions; evaluation uses \(\varepsilon=0\).

\paragraph{Reward model (data \& training).}
We follow the model-based sequence-design setup of \citet{Angermueller2020Model-based}: for each pathogen we train a binary Random-Forest (RF) classifier on GRAMPA that serves as a proxy reward. Preprocessing largely follows prior AMP work \citep{witten2019amp}, with one substantive change tailored to our setting: we restrict peptide lengths to \(1\le L\le 10\). Concretely, we uppercase sequences and strip non-letters; drop YADAMP entries; remove modified peptides via keyword filters (e.g., PEG, fluorophores, lipid tags); harmonize MIC units to \(\mu\)M when present and aggregate replicates by geometric mean per (sequence, bacterium); deduplicate by (sequence, bacterium); and draw negatives as random peptides matched to the positive length histogram (balanced). Sequences are encoded as fixed-width one-hot vectors with \(L_{\max}=10\) and an EOS token at index \(0\) and for each target we fit an Random Forrest.

\paragraph{Runtime scoring.}
Given a peptide \(y\), we score it by the maximum AMP probability across the five RFs (targets:
\emph{E.~coli}, \emph{S.~aureus}, \emph{P.~aeruginosa}, \emph{B.~subtilis}, \emph{C.~albicans}):
\[
p_{\mathrm{RF}}(y)\;=\;\max_{j}\Pr\!\big(f_j(y)=\mathrm{AMP}\big).
\]

\paragraph{Probability-to-reward mapping.}
We choose a probability cutoff \(c=0.94\) and a temperature \(T=0.3\). Sequences with
\(p_{\mathrm{RF}}(y)\ge c\) saturate at unit reward \(R(y)=1\) (equivalently, \(\log R(y)=0\)).
For \(p_{\mathrm{RF}}(y)<c\), we downweight exponentially using a length-aware logistic margin.
Let \(\mathrm{logit}(u)=\log u-\log(1-u)\) and define the margin
\(\Delta(y)=\mathrm{logit}\!\big(p_{\mathrm{RF}}(y)\big)-\mathrm{logit}(c)\).
Then we set
\[
\log R(y)\;=\;
\begin{cases}
0, & p_{\mathrm{RF}}(y)\ge c,\\[4pt]
\displaystyle \frac{L}{T}\,\Delta(y), & p_{\mathrm{RF}}(y)<c,
\end{cases}
\qquad\text{clipped to }[-30,\,0],
\]

\paragraph{Trajectory likelihood and TB-style estimator (trivial backward).}
For a member \(k\) with forward policy \(\pi^{(k)}\) and scalar \(\log Z_k\), and any completed sequence \(y=y_{1:L}\), the TB estimator with the deterministic backward reduces to
\[
\log \widehat{R}^k(y)
\;=\;\log Z_k\;+\;\sum_{t=1}^{L}\log \pi^{(k)}\!\big(a_t \mid y_{1:t-1}\big)
\;=\;\log\!\big(Z_k\,P_F^{(k)}(y)\big).
\]
Because the backward policy is deterministic, the reverse-path probability \(P_B^{(k)}(\tau\mid y)\) is \(1\) along the unique reverse path; hence \(\log \widehat{R}_k(y)\) can be evaluated by replaying the \emph{forward} steps of the current sample—no backward rollouts are required. There is no per-trajectory mixing. In the boosting scheme, the marginal reference reward used by \(\alpha\)-Boost is the sum of terminal flows of the frozen members:
\[
\widehat{R}(y)\;=\;\sum_{k\in\mathcal{F}} Z_k\,P_F^{(k)}(y)\;=\;\sum_{k\in\mathcal{F}} \widehat{R}^k(y).
\]

\paragraph{Training setup and hyperparameters.}
We train with AdamW on batches of on-policy rollouts. The objective is the TB regression in Eq.~\eqref{eq:s1-tb} (or the \(\alpha\)-Boosted variant in Eq.~\eqref{eq:s1-boost}). \emph{Baseline} (\(K{=}1\)) uses a single forward policy \(\pi_\theta\) and a scalar \(\log Z\); the backward policy is deterministic (pop-last-token) and not learned. \emph{Boosting:} we freeze previously trained members and add a new forward policy and scalar trained with the \(\alpha\)-Boosted objective; the reference \(\widehat{R}_\text{old}\) is the sum of terminal flows of the frozen members (Sec.~\ref{sec:s1-boosting}).

\paragraph{Training configuration (peptides).}
\begin{table*}[h]
\centering
\caption{Default training setup for peptide generation.}
\label{tab:s4-train}
\begin{tabular}{@{}l p{0.72\linewidth}@{}}
\hline
Component & Setting \\
\hline
Alphabet \& length & 19 amino acids (no Cys) $+$ STOP; variable length $1\le L\le L_{\max}=10$. \\
Batch size        & 4096. \\
Exploration       & $\varepsilon\in\{0,\,0.1,\,0.2,\,0.3\}$ (evaluation uses $\varepsilon=0$). \\
Seeds             & $\{10,\ldots,14\}$. \\
Policy MLP        & Embedding dim $64$; context window $6$ (last tokens); sinusoidal pos.\ encoding dim $16$; hidden $128$; \\
                  & output over $|\mathcal V|{=}20$ actions; force STOP at $L_{\max}$. Backward is deterministic. \\
Optimizer         & AdamW; learning rates $(\mathrm{pf},\,\log Z)=(5\times 10^{-2},\,10^{-1})$. \\
Schedule          & Baseline: $3000$ epochs; Boosters at checkpoints $1200$ and $2400$ with $\alpha\in\{0,\,1\}$. \\
Reward shaping    & RF cutoff $c=0.94$; temperature $T=0.3$; $\log R$ clipped to $[-30,\,0]$. \\
\hline
\end{tabular}
\end{table*}

\noindent\emph{Seed control.} To prevent subsequent boosters from exploring new areas by coincidence (rather than due to the ensemble’s intrinsic reward), we initialize \textbf{all} boosters with the same seeds as the baseline (same seeds for the policy and environment RNGs).

\paragraph{Evaluation protocol and metric (peptides).}
Every $50$ epochs, we evaluate the current model (or ensemble) as follows. We draw $N{=}1000$ peptide trajectories with $\varepsilon{=}0$ by first sampling an ensemble member $k$ with probability proportional to $Z_k$, then rolling out a trajectory from its forward policy $\pi^{(k)}$. We keep only terminal sequences with $1\le L\le 10$ and deduplicate them to obtain a set $\mathcal{S}_{\mathrm{ep}}$ of unique sequences. Each sequence $y\in\mathcal{S}_{\mathrm{ep}}$ is scored by the RF bundle to obtain $p_{\mathrm{RF}}(y)$; we then form the high-confidence subset
\[
\mathcal{H}_{\mathrm{ep}}\;=\;\bigl\{\,y\in\mathcal{S}_{\mathrm{ep}}:\ p_{\mathrm{RF}}(y)\,\ge\,94\%\,\bigr\}.
\]
Because our reward mapping clips $\log R(y)$ to $0$ whenever $p_{\mathrm{RF}}(y)\ge 94\%$, this is equivalent to selecting sequences with $R(y)=1$. The reported metric at that evaluation point is simply the count $\mathcal{H}_{\mathrm{ep}}$

We report the mean $\pm$ standard deviation across random seeds.

\begin{figure*}[h]
  \centering
  \setlength{\tabcolsep}{2pt}            
  \renewcommand{\arraystretch}{0}         
  \begin{tabular}{*{4}{c}}
    \includegraphics[width=0.23\textwidth]{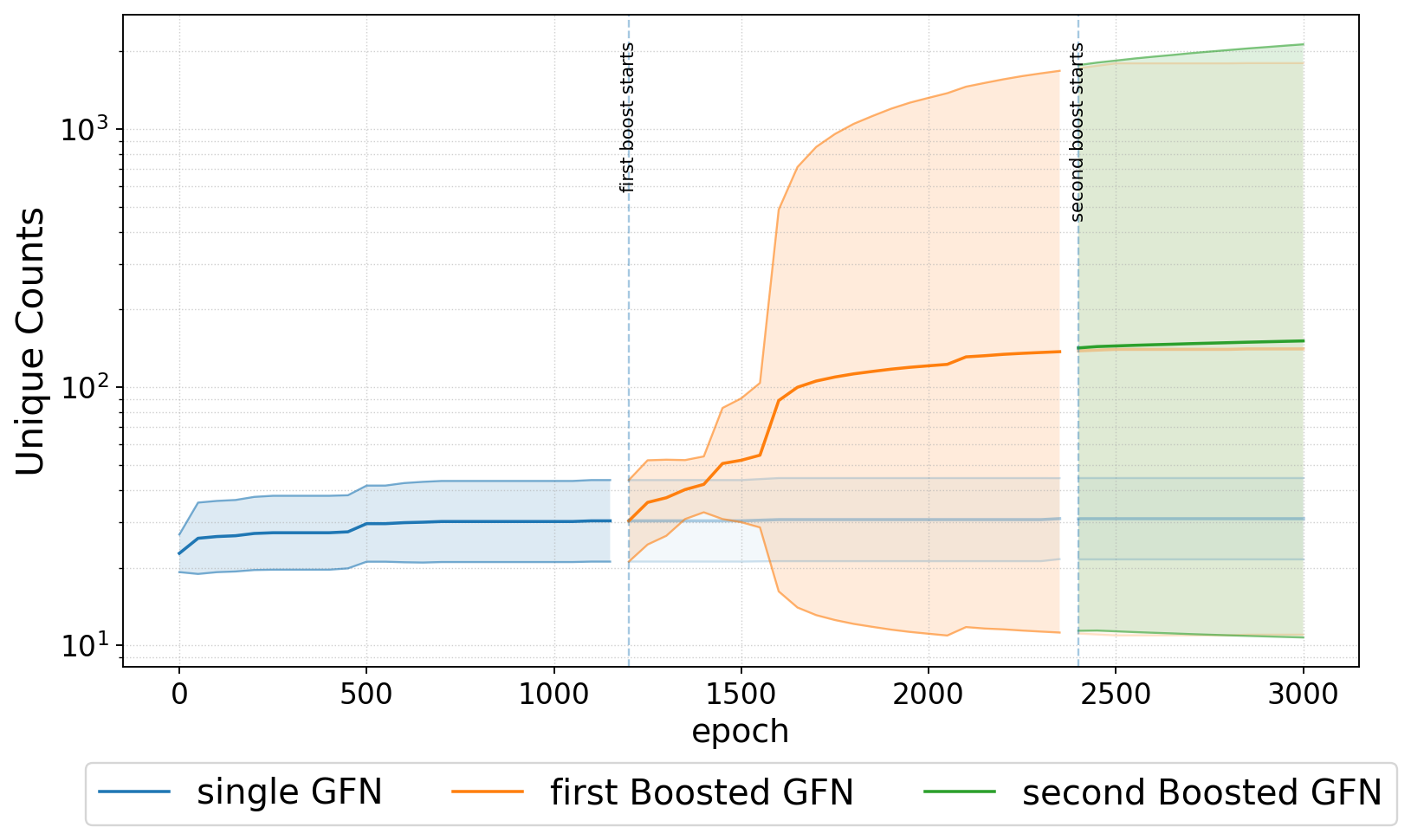} &
    \includegraphics[width=0.23\textwidth]{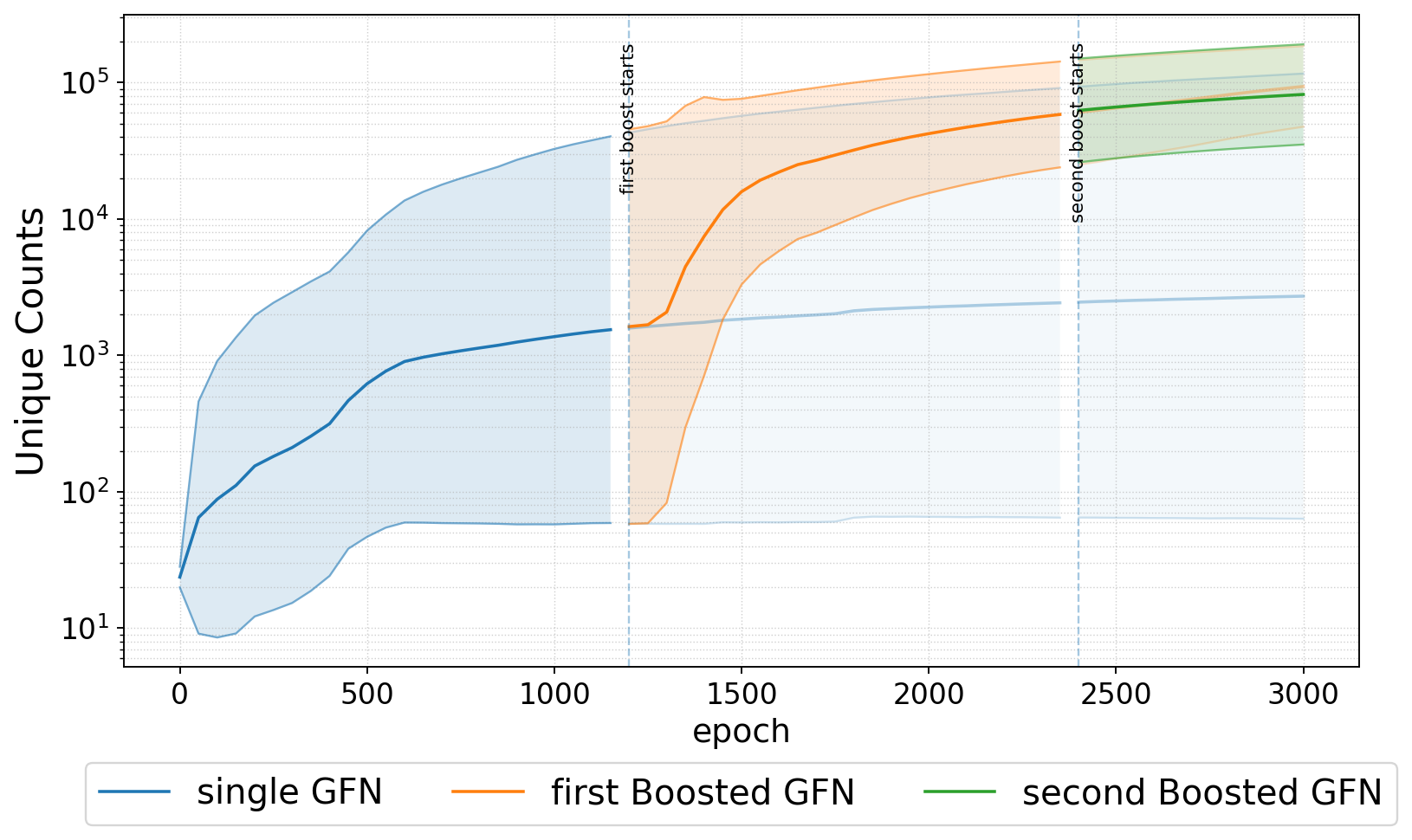} &
    \includegraphics[width=0.23\textwidth]{figs/peptides/alpha-0.0/uniq_ge0_log_eps-0.2.png} &
    \includegraphics[width=0.23\textwidth]{figs/peptides/alpha-0.0/uniq_ge0_log_eps-0.3.png}\\

  \end{tabular}

  \caption{\textbf{Peptide Generation across exploration levels.}
  Columns correspond to $\varepsilon\in\{0,0.1,0.2,0.3\}$.
  Comparing number o unique peptides generated as exploration increases.}
\end{figure*}

\section{Theorems and Proofs}
\label{sec:proofs}
\renewcommand\thetheorem{\arabic{theorem}}
\setcounter{theorem}{0}

\begin{theorem}[Zero variance at the TB optimum]\label{thm:zero-var}
Let \( \Gamma_x := \{\tau:\ \tau \text{ ends at } x\}\) and
\[
P_F(x) := \sum_{\tau\in\Gamma_x} P_F(\tau),\qquad
S := \{x:\ P_F(x)>0\}.
\]
Assume a finite or at-most-countable DAG, so that \(P_F(x)=0 \Rightarrow P_F(\tau)=0\) for all \(\tau\in\Gamma_x\) and assume the TB loss has zero expectation under the forward sampler:
\begin{equation}\label{eq:TB-zero}
\mathbb{E}_{\tau\sim P_F}\!\left[
\left(\log\frac{Z_\theta\,P_F(\tau)}{R(x)\,P_B(\tau\mid x)}\right)^2
\right]=0,
\end{equation}
where \(x\) is the terminal state of \(\tau\).
For every \(x\in S\), assume mutual absolute continuity on \(\Gamma_x\):
\[
P_F(\tau)=0 \iff P_B(\tau\mid x)=0,\qquad \forall\,\tau\in\Gamma_x.
\]
and define $\widehat{R}_\theta(x;\tau) := Z_\theta\,\frac{P_F(\tau)}{P_B(\tau\mid x)}$. Then, for every terminal \(x\):
\[
\mathbb{E}_{\tau\sim P_B(\cdot\mid x)}\!\big[\widehat{R}_\theta(x;\tau)\big]
= R(x)\,\mathbf{1}_{\{x\in S\}},
\qquad
\operatorname{Var}_{\tau\sim P_B(\cdot\mid x)}\!\big[\widehat{R}_\theta(x;\tau)\big]=0.
\]
\end{theorem}

\begin{proof}
From \eqref{eq:TB-zero} we have
\(
Z_\theta P_F(\tau)=R(x)P_B(\tau\mid x)
\)
for \(P_F\)-a.s.\ \(\tau\) (with \(x\) the terminal of \(\tau\)).
If \(x\in S\), mutual absolute continuity on \(\Gamma_x\) transports this equality to
\(P_B(\cdot\mid x)\)-a.s.; if \(x\notin S\), then \(P_F(x)=0\) implies \(P_F(\tau)=0\) for all
\(\tau\in\Gamma_x\), so
\[
Z_\theta P_F(\tau)=R(x)\,\mathbf{1}_{\{x\in S\}}\,P_B(\tau\mid x)
\qquad\text{holds \(P_B(\cdot\mid x)\)-a.s. for every }x.
\]
Dividing by \(P_B(\tau\mid x)\) on a \(P_B(\cdot\mid x)\)-full set yields
\(
\widehat{R}_\theta(x;\tau)=R(x)\,\mathbf{1}_{\{x\in S\}}
\)
\(P_B(\cdot\mid x)\)-a.s., which gives the stated mean and zero variance. \qedhere
\end{proof}
\begin{theorem}[Correctness of the boosted loss]\label{thm:boost-correct}
Fix a terminal state $x$ with $R(x)>0$, and let $\Gamma_x$ be the set of trajectories
ending at $x$. For $\tau\in\Gamma_x$ define
\[
\widehat{R}_\theta(x;\tau) := Z_\theta\,\frac{P_F^\theta(\tau)}{P_B^\theta(\tau\mid x)} ,
\]
and, for a given $\alpha\in(0,1]$ and a learned reward  $\widehat{R}(x)\ge 0$,
consider the per-trajectory boosted loss
\begin{equation}\label{eq:Lboost-def}
L_{\text{boost}}(\theta;\tau,x)
:=\Big(
\log\big[\widehat{R}_\theta(x;\tau)+\alpha\,\widehat{R}(x)\big]
-\log\big[R(x)-(1-\alpha)\,\widehat{R}(x)\big]
\Big)^2 .
\end{equation}
Then:
\begin{enumerate}
\item \textbf{(no degradation)} If $\widehat{R}(x)=R(x)$, then every stationary
point of $L_{\text{boost}}$ (with respect to $\theta$) satisfies $Z_\theta=0$.
\item \textbf{(residual focus)} If $\widehat{R}(x)=0$, then
$L_{\text{boost}}(\theta;\tau,x)$ coincides with the Trajectory Balance (TB)
per-trajectory loss at $x$, and thus the stationary points of $L_{\text{boost}}$
are exactly the stationary points of the TB loss.
\end{enumerate}
\end{theorem}

\begin{proof}
\emph{(ii) Residual focus.}
If $\widehat{R}(x)=0$, then from \eqref{eq:Lboost-def}
\[
L_{\text{boost}}(\theta;\tau,x)
=\Big(\log \widehat{R}_\theta(x;\tau)-\log R(x)\Big)^2
=\Bigg(\log\frac{Z_\theta\,P_F^\theta(\tau)}{R(x)\,P_B^\theta(\tau\mid x)}\Bigg)^2,
\]
which is exactly the TB squared log-ratio at $(\tau,x)$.

\medskip
\emph{(i) No degradation.}
If $\widehat{R}(x)=R(x)$, then
\[
L_{\text{boost}}(\theta;\tau,x)
=\Big(\log\!\big[1+\widehat{R}_\theta(x;\tau)/\alpha R(x)\big]\Big)^2.
\]
Since $u\mapsto \log(1+u)$ is nonnegative and strictly increasing for $u\ge 0$,
the loss is minimized iff $\widehat{R}_\theta(x;\tau)=0$ a.s. On the common support
$\widehat{R}_\theta(x;\tau)= Z_\theta\,\frac{P_F^\theta(\tau)}{P_B^\theta(\tau\mid x)}$ with
$\frac{P_F^\theta(\tau)}{P_B^\theta(\tau\mid x)}>0$ almost surely; hence $Z_\theta=0$ at stationarity.
\end{proof}

\begin{theorem}[Correctness of the sampling process]
Given $N$ stages of GFlowNets $\{(Z_i,P_{i,F},P_{i,B})\}_{i=1}^N$, define
\begin{equation}
    \hat{p}(x) \coloneqq \sum_{i=1}^N \frac{Z_i}{\sum_{j=1}^N Z_j} \sum_{\tau\in\Gamma_x} P_{i,F}(\tau),
\end{equation}
where $\Gamma_x$ denotes the set of trajectories that terminate at $x$.
For a target reward $R(x)$, if $\E_{\tau \sim Q}[L_{\text{boost}}] = 0$ for a distribution $Q$ having full
support over $\Gamma$, then $\hat{p}(x)\propto R(x)$.
\end{theorem}

\begin{proof}
Let $\mathcal X$ be the set of terminal states and $\Gamma$ the set of complete trajectories.
For $x\in\mathcal X$, let $\Gamma_x\subseteq\Gamma$ be the set of trajectories terminating at $x$,
and for $S\subseteq\mathcal X$ define $\Gamma_S:=\bigcup_{x\in S}\Gamma_x$.
For stage $i$, define the terminal marginal
\[
P_i(x)\;:=\;\sum_{\tau\in\Gamma_x} P_{i,F}(\tau),
\qquad x\in\mathcal X.
\]

Because $Q$ has full support over $\Gamma$, we therefore have $L_{\text{boost}}(\tau)=0$ for all $\tau\in\Gamma$,
and hence the defining boosted equality holds on every trajectory.
Let us also consider an ordering of subsets
$S_1,\dots,S_N\subseteq\mathcal X$ and write
\[
U_0:=\varnothing,\qquad U_i:=\bigcup_{j=1}^i S_j.
\]
\noindent\textbf{Stage 1 (classical TB on $S_1$).}
When the TB loss is zero on all trajectories terminating in $S_1$, the trajectory-balance identity holds; that is, for every $x\in S_1$ and every $\tau\in\Gamma_x$,
\begin{equation}
Z_1\,P_{1,F}(\tau)\;=\;R(x)\,P_{1,B}(\tau\mid x).
\label{eq:tbS1_nop}
\end{equation}
Summing \eqref{eq:tbS1_nop} over $\tau\in\Gamma_x$ and using $\sum_{\tau\in\Gamma_x}P_{1,B}(\tau\mid x)=1$ yields
\begin{equation}
Z_1\,P_1(x)\;=\;R(x)\qquad \forall x\in S_1.
\label{eq:flowS1_nop}
\end{equation}
Thus the unnormalized terminal flow contributed by stage $1$ equals $R(x)$ on $S_1$.

\medskip
\noindent\textbf{Inductive claim.}
Assume that after stages $1,\dots,i-1$ the frozen ensemble explains exactly the terminals in $U_{i-1}$, namely
\begin{equation}
\widehat R_{i-1}(x)\;=\;R(x)\,\mathbf 1_{U_{i-1}}(x).
\label{eq:explained_nop}
\end{equation}
We show that stage $i$ contributes unnormalized terminal flow equal to $R(x)$ on the newly covered terminals
$S_i\setminus U_{i-1}$ and $0$ on $S_i\cap U_{i-1}$.

Let $\alpha\in[0,1]$ denote the mixing parameter of the boosted loss, and define
$\widehat R_i(x;\tau):=Z_i\,P_{i,F}(\tau)/P_{i,B}(\tau\mid x)$.
Zero boosted loss implies that for every $x\in S_i$ and every $\tau\in\Gamma_x$,
\begin{equation}
\widehat R_i(x;\tau)+\alpha\,\widehat R_{i-1}(x)
\;=\;
R(x)-(1-\alpha)\,\widehat R_{i-1}(x).
\label{eq:boosteq_nop}
\end{equation}
If $x\in S_i\setminus U_{i-1}$, then $\widehat R_{i-1}(x)=0$ by \eqref{eq:explained_nop}, and
\eqref{eq:boosteq_nop} reduces to $\widehat R_i(x;\tau)=R(x)$ for all $\tau\in\Gamma_x$.
Summing over $\Gamma_x$ gives
\begin{equation}
Z_i\,P_i(x)\;=\;R(x)\qquad \forall x\in S_i\setminus U_{i-1}.
\label{eq:flow_off_nop}
\end{equation}
If instead $x\in S_i\cap U_{i-1}$, then $\widehat R_{i-1}(x)=R(x)$ and \eqref{eq:boosteq_nop} gives
$\widehat R_i(x;\tau)=0$ for all $\tau\in\Gamma_x$, hence
\begin{equation}
Z_i\,P_i(x)\;=\;0\qquad \forall x\in S_i\cap U_{i-1}.
\label{eq:flow_on_nop}
\end{equation}
Therefore, for all $x\in\mathcal X$,
\begin{equation}
Z_i\,P_i(x)\;=\;R(x)\,\mathbf 1_{S_i\setminus U_{i-1}}(x).
\label{eq:stagei_flow_nop}
\end{equation}
In particular, adding stage $i$ extends the explained set from $U_{i-1}$ to $U_i$ and preserves the form
\eqref{eq:explained_nop} with $i$ in place of $i-1$, completing the inductive step.

Summing \eqref{eq:stagei_flow_nop} over $i=1,\dots,N$ then gives, for every $x\in\mathcal X$,
\begin{equation}
\sum_{i=1}^N Z_i\,P_i(x)
\;=\;
R(x)\sum_{i=1}^N \mathbf 1_{S_i\setminus U_{i-1}}(x)
\;=\;
R(x)\,\mathbf 1_{U_N}(x).
\label{eq:sumflows_nop}
\end{equation}
In particular, if $U_N=\mathcal X$, we have $\sum_{i=1}^N Z_i\,P_i(x)=R(x)$ for all $x\in\mathcal X$.

Finally, consider the sampling procedure that first draws an index $I$ with
$P(I=i)=Z_i/\sum_{j=1}^N Z_j$ and then samples $x\sim P_i(\cdot)$ by running the stage-$i$ forward policy.
The induced terminal distribution is
\[
\hat p(x)
=\sum_{i=1}^N \frac{Z_i}{\sum_{j=1}^N Z_j} P_i(x)
=\frac{1}{\sum_{j=1}^N Z_j}\sum_{i=1}^N Z_i\,P_i(x).
\]
Applying \eqref{eq:sumflows_nop} with $U_N=\mathcal X$ yields
\[
\hat p(x)=\frac{R(x)}{\sum_{y\in\mathcal X} R(y)},
\]
and therefore $\hat p(x)\propto R(x)$, as claimed.
\end{proof}

\setcounter{proposition}{1 - 1}

\begin{proposition} \label{prop:bvis}
    Let $\widehat Z$ and $Z_{\theta}$ be the normalizing constants for $\widehat {R}(x) \coloneqq \mathbb{E}_{\tau \sim P_B(\cdot \mid x)}[\widehat{Z} \frac{\widehat P_F(\tau)}{P_B(\tau \mid x)}]$ and ${R}_{\theta}(x; \tau ) \coloneq Z_\theta \frac{P_F(\tau)}{P_B(\tau \mid x)}$. 
    Also, let $\beta = \nicefrac{Z_{\theta}}{\widehat Z + Z_{\theta}}$ be the mixture weight in \Cref{eq:bvis}. 
    Then, define $\widehat p(\tau) \propto \widehat{R}(x) p_B(\tau | x)$ and $p_{\mathrm{tgt}}(\tau) = R(x) p_B(\tau | x)$ as the trajectory-level distributions induced by our current and target models, and let $p_M(\tau) = (1 - \beta) \cdot \widehat p(\tau) + \beta \cdot p_{F}(\tau)$ be the corresponding mixture distribution.
    In this scenario,
    \begin{equation} 
        \mathbb{E}_{p_M} \left[ \nabla_{\theta} L_{\mathrm{boost}}(\tau) \right] = 2 \cdot \nabla_{\theta} D_{\mathrm{KL}} [ p_M(\tau) || p_{\mathrm{tgt}}(\tau) ]. 
    \end{equation}
\end{proposition}

\begin{proof}
    Firstly, we notice that $L_{\mathrm{boost}}$ can be re-written as 
    \begin{equation}
        L_{\mathrm{boost}}(\tau) =  \left(\log  \frac{{R}(x ; \theta) + \widehat{R}(x)}{R(x)} \right)^{2}.
    \end{equation} 
    Since $R(x; \tau) = Z_{\theta} \cdot \frac{P_{F}(\tau)}{P_{B}(\tau | x)}$, this can also be represented as 
    \begin{equation}
        L_{\mathrm{boost}}(\tau) =  \left(\log  \frac{Z_{\theta} p_{F}(\tau) + \widehat{R}(x) p_{B}(\tau | x) }{R(x) p_{B}(\tau | x)} \right)^{2}.
    \end{equation} 
    By definition, $\widehat Z$ is the normalizing constant for $\widehat{R}(x)$. Then, under the notations of our proposition, 
    \begin{equation}
        L_{\mathrm{boost}}(\tau) = \left(\log  \frac{\beta p_{F}(\tau) + (1 - \beta) \widehat p(\tau) }{\frac{R(x)}{\widehat Z + Z_{\theta}} p_{B}(\tau | x)} \right)^{2} = 
        \left(\log  \frac{\beta p_{F}(\tau) + (1 - \beta) \widehat p(\tau) }{p_{\mathrm{tgt}}(\tau)} \right)^{2}.
    \end{equation} 
    All in all, $L_{\mathrm{boost}}(\tau) = \left(\log  \frac{p_{M}(\tau)}{p_{\mathrm{tgt}}(\tau)} \right)^{2}$. From this perspective, 
    \begin{equation}
        \nabla_{\theta} L_{\mathrm{boost}}(\tau) = \nabla_{\theta} \left( \log \frac{p_{M}(\tau)}{p_{\mathrm{tgt}}(\tau)} \right)^{2} = 2 \left ( \log \frac{p_{M}(\tau)}{p_{\mathrm{tgt}}(\tau)} \right) \nabla_{\theta} \log p_{M}(\tau); 
    \end{equation}
    the second equality relies on the fact that $p_{\mathrm{tgt}}$ does not depend on $\theta$ (it does depend on $Z_{\theta}$, which is a parameter independent of $\theta$). 
    Similarly, by Leibniz's integral rule, 
    \begin{equation}
        \nabla_{\theta} \mathcal{D}_{\mathrm{KL}} [ p_M  || p_{\mathrm{tgt}} ] = \nabla_{\theta} \mathbb{E}_{\tau \sim p_{M}} \left[ \log \frac{p_{M}(\tau)}{p_{\mathrm{tgt}}(\tau)} \right] = \mathbb{E}_{\tau \sim p_{M}} \left[ \nabla_{\theta} \log \frac{p_{M}(\tau)}{p_{\mathrm{tgt}}(\tau)} +  \left( \log \frac{p_{M}(\tau)}{p_{\mathrm{tgt}}(\tau)} \right) \cdot \nabla_{\theta} \log p_{M}(\tau) \right]. 
    \end{equation}
    Clearly, since $p_{\mathrm{tgt}}$ does not depend on $\theta$, 
    \begin{equation}
    \begin{aligned}
        \nabla_{\theta} \mathcal{D}_{\mathrm{KL}} [ p_M || p_{\mathrm{tgt}} ]
        &= \underset{{\tau \sim p_{M}}}{\mathbb{E}} \left[ \nabla_{\theta} \log p_{M}(\tau) + \left( \log \frac{p_{M}(\tau)}{p_{\mathrm{tgt}}(\tau)} \right) \cdot \nabla_{\theta} \log p_{M}(\tau) \right] \\
        &= \underset{{\tau \sim p_{M}}}{\mathbb{E}} \left[ \left( \log \frac{p_{M}(\tau)}{p_{\mathrm{tgt}}(\tau)} \right) \cdot \nabla_{\theta} \log p_{M}(\tau) \right],
    \end{aligned}
    \end{equation}
    in which we used the fact that $\mathbb{E}_{\tau \sim p} [ \log p(\tau) ] = 0$ for any probability distribution $p$ \citep{williams1992}.  
    
\end{proof}

\end{document}